%% file: main.tex
\documentclass{article}

\usepackage{iclr2027_conference,times}
\usepackage[T1]{fontenc}
\usepackage[utf8]{inputenc}
\usepackage{microtype}
\usepackage{needspace}
\usepackage{amsmath,amssymb,amsfonts,amsthm,mathtools,bm}
\usepackage{aliascnt}
\usepackage{booktabs,tabularx,array,float}
\usepackage{algorithm,algpseudocode}
\usepackage{graphicx}
\usepackage[caption=false,font=footnotesize]{subfig}
\usepackage{arydshln}
\usepackage{xcolor}
\usepackage[colorlinks=true,allcolors=blue]{hyperref}
\usepackage[capitalize,noabbrev,nameinlink]{cleveref}

\iclrfinalcopy

\theoremstyle{plain}
\newtheorem{theorem}{Theorem}[section]
\newaliascnt{lemma}{theorem}
\newtheorem{lemma}[lemma]{Lemma}
\aliascntresetthe{lemma}
\newaliascnt{proposition}{theorem}
\newtheorem{proposition}[proposition]{Proposition}
\aliascntresetthe{proposition}
\newaliascnt{corollary}{theorem}

\aliascntresetthe{corollary}
\theoremstyle{definition}
\newaliascnt{definition}{theorem}

\aliascntresetthe{definition}
\newaliascnt{assumption}{theorem}
\newtheorem{assumption}[assumption]{Assumption}
\aliascntresetthe{assumption}
\crefname{assumption}{Assumption}{Assumptions}
\Crefname{assumption}{Assumption}{Assumptions}
\theoremstyle{plain}
\newaliascnt{remark}{theorem}

\aliascntresetthe{remark}
\newtheorem*{remark*}{Remark}
\usepackage{thmtools,thm-restate}

\numberwithin{equation}{section}
\renewcommand{\theHequation}{\thesection.\arabic{equation}}

\newcommand{\R}{\mathbb{R}}
\newcommand{\E}{\mathbb{E}}
\newcommand{\Class}{\mathcal C}
\newcommand{\norm}[1]{\left\lVert #1\right\rVert}
\newcommand{\inner}[2]{\left\langle #1,#2\right\rangle}
\newcommand{\ip}[2]{\inner{#1}{#2}}
\newcommand{\trans}{\mathsf T}
\newcommand{\eps}{\varepsilon}
\DeclareRobustCommand{\SGDAAlt}{SGDA\textsubscript{Alt}}
\DeclareRobustCommand{\SGDASim}{SGDA\textsubscript{Sim}}
\newcolumntype{Y}{>{\raggedright\arraybackslash}X}
\DeclareMathOperator{\dist}{dist}
\DeclareMathOperator*{\argmax}{arg\,max}

\title{Stochastic Gradient Descent Ascent is\\
Suboptimal for Nonconvex-P{\L} Min-Max Games}
\author{Junsoo Ha \\
\texttt{junsoo.ha.contact@gmail.com}}
\hypersetup{
  pdftitle={Stochastic Gradient Descent Ascent is Suboptimal for Nonconvex-P{\L} Min-Max Games},
  pdfauthor={Junsoo Ha}
}
\begin{document}
% Restrict the table of contents below to the appendix.
\addtocontents{toc}{\protect\setcounter{tocdepth}{-1}}
\maketitle

\begin{abstract}
How far can stochastic gradient descent ascent (SGDA) go by tuning its timescale ratio and step sizes in nonconvex min-max games? We answer this question for nonconvex-P{\L} (NC-P{\L}) games by establishing the first tight complexity of two-timescale SGDA with a fixed timescale ratio and non-increasing step sizes. For $\ell$-smooth games with an inner $\mu$-P{\L} inequality, we prove a complexity lower bound $\Omega(\kappa^2\ell\eps^{-2}+\kappa^4\ell\sigma^2\eps^{-4})$, where $\kappa=\ell/\mu$ is the condition number, $\sigma^2$ is the gradient variance, and $\eps$ measures the outer gradient norm. This matches existing SGDA upper bounds and establishes a complexity separation from Smoothed-AGDA \citep{yang2022faster}. In addition, we show that SGDA can fail to find a stationary point when its timescale ratio is as small as $o(\kappa^2)$. Our negative results highlight the fundamental limitation of SGDA in NC-P{\L} games, and justify the development of alternative methods.
\end{abstract}

\input{sections/introduction}
\input{sections/related_work}
\input{sections/setup}
% Defer the comparison table to the top of page 4.
\suppressfloats[t]
\input{sections/comparison_table}
\input{sections/main_results}
\input{sections/high_ratio}
\input{sections/low_ratio}
\input{sections/conclusion}
% The editing preview copies pages through this marker from the full PDF.
\label{page:main-text-end}

% ICLR requires this statement but excludes it from the main-text limit.
\clearpage
\section*{AI Use Statement}
Generative AI tools were used to brainstorm and refine theoretical constructions, formulate and check proof strategies, survey relevant literature, draft and edit mathematical exposition, and prepare and check the \LaTeX{} source. The authors are responsible for the final content, including all AI-assisted text, claims, proofs, and artifacts.

\section*{Reproducibility Statement}
All assumptions and algorithms are stated in \cref{sec:setup}. Appendix~\ref{app:main-result} contains the proofs of the main theorem, the high- and low-ratio lower bounds, and the matching \SGDASim{} upper bound. Appendices~\ref{app:high-ratio} and~\ref{app:low-ratio} prove the construction lemmas. Together, these proofs give the construction formulas, parameter choices, oracle definitions, initialization checks, and arguments for \SGDASim{} and \SGDAAlt{}.

\bibliographystyle{iclr2027_conference}
\bibliography{references}

\clearpage
\appendix
\addtocontents{toc}{\protect\setcounter{tocdepth}{2}}
\begingroup
\renewcommand{\contentsname}{Appendix contents}
\tableofcontents
\endgroup
\clearpage
\input{appendices/main_result}

\input{appendices/preliminaries}
\input{appendices/high_ratio_proofs}
\input{appendices/low_ratio_proofs}
% Focused comparison supporting the output and initialization alignment in Section 4.
\input{appendices/smoothed_comparison}
\input{appendices/table_comparison}
% Broader historical comparison notes remain uncompiled: appendices/related_work.tex.

\end{document}

%% file: sections/introduction.tex
\section{Introduction}
\label{sec:introduction}

Nonconvex min-max games have become a cornerstone of modern machine learning, underpinning generative adversarial networks (GANs) \citep{goodfellow2014generative}, distributionally robust learning \citep{sinha2018certifying}, and language model alignment \citep{liu2026generative,xu2025robust}. Without additional structure, however, even finding a stationary point of a min-max game can be computationally intractable \citep{daskalakis2021complexity,bernasconi2026complexity}. For instance, games often exhibit limit cycles \citep{hsieh2021limits} and chaotic behavior \citep{sato2002chaos}. Such pathological examples motivate the study of \emph{structured} nonconvex min-max games \citep{lin2020gda,lin2025two,ha2026high}. In particular, nonconvex-P{\L} (NC-P{\L}) games assume a Polyak--\L{}ojasiewicz (P{\L}) inequality \citep{polyak1963gradient} on the inner maximization problem and have been actively studied as one of the weakest forms of structure \citep{nouiehed2019solving,yang2022faster,cho2023sgda,ha2026high,cai2026bias}.

Even in NC-P{\L} games, finding a Nash equilibrium \citep{nash1951noncooperative} can be still intractable; hence, we seek a stationary point of the outer function $\Phi(x):=\max_y f(x,y)$. Two-timescale SGDA is a direct analog of SGD for such min-max games: it uses a faster ascent step to track the outer gradient $\nabla\Phi(x)$ and performs a slower inexact descent on $\Phi$. For $\ell$-smooth nonconvex-$\mu$-P{\L} games with condition number $\kappa=\ell/\mu$ and gradient variance $\sigma^2$, SGDA admits a complexity upper bound $O(\kappa^2\ell\eps^{-2}+\kappa^4\ell\sigma^2\eps^{-4})$ for finding an $\eps$-stationary point \citep{yang2022faster,ha2026high}. Meanwhile, recent studies have developed more sophisticated techniques to attain sharper bounds \citep{yang2022faster,huang2025enhanced,cai2026bias}. This raises the following natural question:
\begin{center}
\itshape
Can SGDA obtain better complexity by tuning its timescale ratio and step sizes?
\end{center}

\Needspace{6\baselineskip}
\paragraph{Contributions.}
In this work, we answer this question negatively. Our contributions are twofold.
\begin{enumerate}
\renewcommand{\labelenumi}{\textup{(\roman{enumi})}}
\item \textbf{Suboptimality of SGDA.} We prove matching lower bounds for SGDA with a fixed timescale ratio and nonincreasing step sizes, and establish its suboptimality in NC-P{\L} games.
\item \textbf{Optimal timescale ratio.} Additionally, we prove that SGDA with a small timescale ratio $o(\kappa^2)$ can fail to find a stationary point, and that existing $\Theta(\kappa^2)$ timescale ratio is optimal.
\end{enumerate}

Sections~\ref{sec:related-work}--\ref{sec:main-results} cover related work, the setup, and main results; Sections~\ref{sec:high-ratio}--\ref{sec:low-ratio} prove the lower bounds; Section~\ref{sec:conclusion} discusses implications and future directions.

%% file: sections/related_work.tex
\section{Related work}
\label{sec:related-work}

\textbf{Nonconvex games}
are computationally intractable in general \citep{daskalakis2021complexity}, and even finding an approximate stationary point can be hard as well \citep{bernasconi2026complexity}. Such hardness results motivate structured games, including nonconvex-P{\L} (NC-P{\L}) \citep{nouiehed2019solving,yang2020global,yang2022faster,cho2023sgda} and nonconvex-concave (NC-C) games \citep{thekumparampil2019efficient,lin2020gda,lin2025two,ha2026high}, which assume a P{\L} inequality and concavity on inner maximization problems, respectively. Notably, these structures make finding stationary points computationally tractable and admit inexact-descent interpretations \citep{nouiehed2019solving,lin2020gda,lin2025two,ha2026high}. In particular, \citet{nouiehed2019solving} showed that taking multiple ascent steps per descent step suffices to find an approximate stationary point in NC-P{\L} games, and interpreted these updates as inexact descent on $\Phi(x)=\max_y f(x,y)$. \citet{lin2020gda,lin2025two} extended the same interpretation to two-timescale SGDA for nonconvex-strongly-concave (NC-SC) and NC-C games. On the other hand, weakly-convex-weakly-concave games \citep{liu2021first,grimmer2023landscape,hajizadeh2024linear} and nonmonotone variational inequalities \citep{mertikopoulos2019optimistic,diakonikolas2021efficient,lee2021fast,pethick2022escaping,bohm2023solving} have been studied as well, focusing on extragradient and optimistic gradient (OGDA), rather than SGDA.

\textbf{NC-P{\L} games.} \citet{yang2020global} gave the first convergence analysis of two-timescale alternating gradient descent ascent (GDA) for NC-P{\L} games. Subsequently, \citet{yang2022faster} extended the analysis to stochastic settings and introduced Smoothed-AGDA, attaining $O(\kappa\ell\eps^{-2}+\kappa^2\ell\sigma^2\eps^{-4})$ under the comparison conditions in Appendix~\ref{app:smoothing-alignment}. Meanwhile, \citet{huang2025enhanced} used momentum to attain an improved leading stochastic complexity of $\widetilde O((\kappa^3\ell\Delta\sigma+\kappa^4\sigma^3)\eps^{-3})$ under Lipschitz sample gradient assumptions; \citet{cai2026bias} obtained the same dependence on $\eps$ using Hessian-vector products to correct momentum bias. For finite-sum NC-P{\L} games, \citet{chen2022faster} and \citet{cho2023sgda} obtained better complexities through variance reduction and sampling without replacement, respectively. Finally, recent studies also established high-probability guarantees for Smoothed-AGDA \citep{laguel2024high} and SGDA \citep{ha2026high}.

\textbf{Two-timescale SGDA} exposes a timescale ratio hyperparameter $\rho$ that controls the size of ascent steps relative to descent steps.
\citet{fiez2021local} showed that a sufficiently large timescale ratio $\rho$ yields local linear convergence of simultaneous GDA to differential Stackelberg equilibria. For NC-SC games, \citet{li2022local} proved that $\rho=\Theta(\kappa)$ suffices for local convergence, whereas GDA with $\rho=o(\kappa)$ can fail to converge locally. On the other hand, existing global convergence guarantees for GDA and SGDA in NC-P{\L} games rely on a timescale ratio $\rho=\Theta(\kappa^2)$ \citep{fiez2021global,yang2022faster,ha2026high}. \citet{cho2023sgda} showed that an excessively large timescale ratio can worsen convergence rates of GDA in strongly-convex-strongly-concave (SC-SC) games, and raised the question whether $\rho=\Theta(\kappa^2)$ is necessary for global convergence in NC-P{\L} games.

\textbf{Lower bounds} for structured nonconvex min-max games have mainly addressed the minimax complexity of first-order methods. For instance, in NC-SC games, \citet{zhang2021complexity} proved a deterministic $\Omega(\sqrt{\kappa}\ell\Delta\eps^{-2})$ lower bound, and \citet{li2021complexity} proved a stochastic $\Omega(\sqrt{\kappa}\ell\Delta\eps^{-2}+\kappa^{1/3}\ell\Delta\sigma^2\eps^{-4})$ bound. More recently, \citet{pan2026lower} proved a sharper deterministic NC-P{\L} lower bound $\Omega(\kappa\ell\Delta\eps^{-2})$, establishing the minimax optimality of Smoothed-AGDA's leading complexity. In this work, we derive SGDA-specific lower bounds that match existing upper bounds \citep{yang2022faster,ha2026high}, and establish provable suboptimality of SGDA in NC-P{\L} games.

%% file: sections/setup.tex
\section{Problem setup}
\label{sec:setup}

We study the following min-max optimization problem:
\begin{equation}
 \min_{x\in\R^{d_x}}\max_{y\in\R^{d_y}} f(x,y),
 \label{eq:problem}
\end{equation}
where $f$ satisfies a smoothness assumption and $f(x,\cdot)$ satisfies a P{\L} inequality \citep{polyak1963gradient}.
\begin{assumption}[Smoothness]
\label{ass:smoothness}
The game $f:\R^{d_x}\times\R^{d_y}\to\R$ is differentiable and $\ell$-smooth:
\begin{align*}
 \norm{\nabla_x f(x_1,y_1)-\nabla_x f(x_2,y_2)}
 &\le\ell\bigl(\norm{x_1-x_2}+\norm{y_1-y_2}\bigr),
 &&\forall x_1,x_2\in\R^{d_x},\ y_1,y_2\in\R^{d_y},\\
 \norm{\nabla_y f(x_1,y_1)-\nabla_y f(x_2,y_2)}
 &\le\ell\bigl(\norm{x_1-x_2}+\norm{y_1-y_2}\bigr),
 &&\forall x_1,x_2\in\R^{d_x},\ y_1,y_2\in\R^{d_y}.
\end{align*}
\end{assumption}
We denote the outer objective of a game by $\Phi(x):=\max_{y\in\R^{d_y}}f(x,y)$, and let $\Phi^*:=\inf_{x\in\R^{d_x}}\Phi(x)$. We define the condition number $\kappa:=\ell/\mu\ge1$, where $\mu>0$ is the P{\L} parameter.
\begin{assumption}[NC-P{\L}]
\label{ass:nc-pl}
The outer objective satisfies $\Phi(x)<\infty$ for every $x\in\R^{d_x}$, and $Y^\star(x):=\argmax_{y\in\R^{d_y}}f(x,y)$ is nonempty for every $x\in\R^{d_x}$. There exists $\mu>0$ such that the following inequality holds for each $x\in\R^{d_x}$:
\begin{equation}
 \frac12\norm{\nabla_y f(x,y)}^2\ge \mu\bigl(\Phi(x)-f(x,y)\bigr),
 \qquad\forall y\in\R^{d_y}.
 \label{eq:inner-pl}
\end{equation}
\end{assumption}
\begin{assumption}[Initialization]
\label{ass:initialization}
There exists $\Delta>0$ and a universal constant $c_{\rm init}\ge1$ such that
\begin{equation}
 0\le\Phi(x_0)-\Phi^*\le\Delta,
 \qquad
 0\le\Phi(x_0)-f(x_0,y_0)\le c_{\rm init}\Delta.
 \label{eq:gap-init}
\end{equation}
\end{assumption}
Throughout the paper, we make the following standard stochastic gradient estimator assumptions.
\begin{assumption}[Unbiasedness]
\label{ass:unbiasedness}
The gradient estimators $G_x,G_y$ are measurable and use fresh independent samples $\xi$ from a fixed distribution. For every $(x,y)$,
\begin{equation}
 \E G_x(x,y;\xi)=\nabla_x f(x,y),
 \qquad
 \E G_y(x,y;\xi)=\nabla_y f(x,y).
 \label{eq:unbiased-oracle}
\end{equation}
\end{assumption}
\begin{assumption}[Bounded variance]
\label{ass:bounded-variance}
For a fixed $\sigma\ge0$ and every $(x,y)$,
\begin{equation}
 \E\norm{G_x(x,y;\xi)-\nabla_x f(x,y)}^2\le\sigma^2,
 \qquad
 \E\norm{G_y(x,y;\xi)-\nabla_y f(x,y)}^2\le\sigma^2.
 \label{eq:oracle}
\end{equation}
\end{assumption}
\paragraph{Two-timescale SGDA.}
SGDA with a fixed timescale ratio $\rho\ge1$ and step sizes $\eta_t\ge0$ follows:
\begin{equation}
 x_{t+1}=x_t-\frac{\eta_t}{\rho}G_x(x_t,y_t;\xi_t^x),
 \qquad
 y_{t+1}=y_t+\eta_tG_y(x'_t,y_t;\xi_t^y),
 \label{eq:updates}
\end{equation}
where we use $x'_t=x_{t+1}$ for \SGDAAlt{}, and $x'_t=x_t$ for \SGDASim{}.

\paragraph{Minimax complexity.}
\looseness=-1 Let $\Class(\ell,\kappa,\Delta,\sigma)$ collect all NC-P{\L} game instances that satisfy Assumptions~\ref{ass:smoothness}--\ref{ass:bounded-variance} with $d_x,d_y<\infty$. For each SGDA update order $\mathsf o\in\{\mathrm{Alt},\mathrm{Sim}\}$, we define the minimax convergence rate of SGDA with a fixed timescale ratio $\rho$ and nonincreasing step sizes $\eta_t$ as
\begin{equation}
 \mathcal R_{T,\mathsf o}(\rho)
 :=\inf_{\eta_0\ge\cdots\ge\eta_{T-2}\ge0}
 \sup_{I\in\Class(\ell,\kappa,\Delta,\sigma)}
 \frac1T\sum_{t=0}^{T-1}\E\norm{\nabla\Phi(x_t)}^2.
 \label{eq:minimax-risks}
\end{equation}
We write the version optimized over the timescale ratio as $\mathcal R_{T,\mathsf o}^\star:=\inf_{\rho\ge1}\mathcal R_{T,\mathsf o}(\rho)$. We emphasize that we allow parameter-dependent tuning of timescale ratio $\rho(\ell,\kappa,\Delta,\sigma,T)$ and step sizes $\eta_t(\ell,\kappa,\Delta,\sigma,T)$. Finally, we define the corresponding iteration complexities as
\begin{equation}
 N_{\eps,\mathsf o}(\rho):=\inf\bigl\{T\in\mathbb N_{\ge1}:\mathcal R_{T,\mathsf o}(\rho)\le\eps^2\bigr\},
 \qquad
 N_{\eps,\mathsf o}^\star:=\inf\bigl\{T\in\mathbb N_{\ge1}:\mathcal R_{T,\mathsf o}^\star\le\eps^2\bigr\}.
\label{eq:iteration-complexities}
\end{equation}

%% file: sections/comparison_table.tex
\begin{table}[t]
\definecolor{tableblue}{HTML}{0066FF}
\definecolor{tablepink}{HTML}{E00070}
\caption{\textbf{Upper and lower bounds in NC-P{\L} games.} Our lower bounds match SGDA's upper bounds after optimizing fixed timescale ratios and predetermined nonincreasing step sizes. With exact gradients, Theorem~\ref{thm:low-ratio} rules out uniform convergence for $\rho=o(\kappa^2)$; $\infty$ denotes nonvanishing stationarity error. MSGDA and AdaMSGDA use $0<\eps\le\sigma$. $\dagger$ AdaMSGDA measures the average expected squared norm of $A_t^{-1}\nabla\Phi(x_t)$, where $A_t$ is its primal preconditioner. $\ddagger$ HCMM-2 shows only the variance-dependent cost for $\E\norm{\nabla\Phi}\le\eps$, with Hessian variance at most $\sigma^2$. Appendix~\ref{app:table-comparison} gives the complete bounds and assumptions.}
\label{tab:ncpl-landscape}
\centering
\begingroup
\setlength{\tabcolsep}{3pt}
\setlength{\arrayrulewidth}{0.3pt}
\begin{tabular*}{\textwidth}{@{\extracolsep{\fill}}llll@{}}
\toprule
Algorithm & Citation & Oracle & Complexity / obstruction \\
\midrule
\SGDAAlt{} & \citet{yang2022faster} & Bounded $\sigma^2$ & $O(\kappa^2\ell\eps^{-2}+\kappa^4\ell\sigma^2\eps^{-4})$ \\
\SGDASim{} & \cref{thm:sim-upper} & Bounded $\sigma^2$ & $O(\kappa^2\ell\eps^{-2}+\kappa^4\ell\sigma^2\eps^{-4})$ \\
Smoothed-AGDA & \citet{yang2022faster} & Bounded $\sigma^2$ & $\widetilde O(\textcolor{tableblue}{\bm{\kappa}}\ell\eps^{-2}+\textcolor{tableblue}{\bm{\kappa^2}}\ell\sigma^2\eps^{-4})$ \\
SPIDER-GDA & \citet{chen2022faster} & Finite-sum & $O(n+\sqrt n\kappa^2\ell\eps^{-2})$ \\
SGDA-RR & \citet{cho2023sgda} & Finite-sum & $O(n\kappa^2\ell\eps^{-2}+\sqrt n\kappa^3\ell\sigma\eps^{-3})$ \\
MSGDA & \citet{huang2025enhanced} & Smooth samples & $\widetilde O((\kappa^3\ell\sigma+\kappa^4\sigma^3)\textcolor{tableblue}{\bm{\eps^{-3}}})$ \\
AdaMSGDA$^{\dagger}$ & \citet{huang2025enhanced} & Smooth samples & $\widetilde O((\kappa^3\ell\sigma+\kappa^4\sigma^3)\textcolor{tableblue}{\bm{\eps^{-3}}})$ \\
HCMM-2$^{\ddagger}$ & \citet{cai2026bias} & Hessian products & $O_\ell(\kappa^3\sigma^3\textcolor{tableblue}{\bm{\eps^{-3}}})$ \\
\hdashline[2pt/2pt]
\noalign{\vskip2pt}
\textbf{\SGDAAlt{}, \SGDASim{}} & \textbf{\cref{thm:common-lower}} & \textbf{Bounded $\bm{\sigma^2}$} & $\Theta(\textcolor{tablepink}{\bm{\kappa^2}}\ell\eps^{-2}+\textcolor{tablepink}{\bm{\kappa^4}}\ell\sigma^2\textcolor{tablepink}{\bm{\eps^{-4}}})$ \\
\midrule
Linear-span (NC-SC) & \citet{zhang2021complexity} & Deterministic & $\Omega(\sqrt{\kappa}\ell\eps^{-2})$ \\
Linear-span (NC-SC) & \citet{zhang2021complexity} & Finite-sum & $\Omega(n+\sqrt{n\kappa}\ell\eps^{-2})$ \\
Zero-respecting (NC-SC) & \citet{li2021complexity} & Bounded $\sigma^2$ & $\Omega(\sqrt{\kappa}\ell\eps^{-2}+\kappa^{1/3}\ell\sigma^2\eps^{-4})$ \\
First-order (NC-P{\L}) & \citet{pan2026lower} & Deterministic & $\Omega(\kappa\ell\eps^{-2})$ \\
\SGDASim{} & \citet{li2022local} & Deterministic & $\infty$ if $\rho=o(\kappa)$ \\
\hdashline[2pt/2pt]
\noalign{\vskip2pt}
\textbf{\SGDAAlt{}, \SGDASim{}} & \textbf{\cref{thm:low-ratio}} & \textbf{Deterministic} & $\infty$ if $\rho=o(\kappa^2)$ \\
\textbf{\SGDAAlt{}, \SGDASim{}} & \textbf{\cref{thm:high-ratio}} & \textbf{Bounded $\bm{\sigma^2}$} & $\Omega(\rho\ell\eps^{-2}+\kappa^4\ell\sigma^2\eps^{-4})$ \\
\bottomrule
\end{tabular*}
\endgroup
\end{table}

%% file: sections/main_results.tex
\section{Main results}
\label{sec:main-results}

Our main result establishes a tight lower bound for two-timescale SGDA with a fixed timescale ratio $\rho\ge1$ and non-increasing step size schedules $\eta_t$.

\begin{restatable}[A lower bound for two-timescale SGDA in NC-P{\L} games]{theorem}{commonLowerStatement}
\label{thm:common-lower}
There are universal constants $c>0$ and $\kappa_0\ge2$, depending on $c_{\rm init}$, such that, for every $\kappa\ge\kappa_0$, $\ell>0$, $\Delta>0$, $\sigma\ge0$, and integer $T\ge1$,
\begin{equation}
 \mathcal R_{T,\mathrm{Sim}}^\star, \mathcal R_{T,\mathrm{Alt}}^\star
 \ge c\min\!\left\{\ell\Delta,
 \max\!\left\{\frac{\kappa^2\ell\Delta}{T},
 \sqrt{\frac{\kappa^4\ell\Delta\sigma^2}{T}}\right\}\right\}.
 \label{eq:common-lower}
\end{equation}
\end{restatable}

Intuitively, the first term $\ell\Delta$ inside the minimum accounts for short horizons. Hence, in the usual small-error regime $\eps^2\le c\ell\Delta$, \cref{thm:common-lower} and the matching upper bounds for \SGDAAlt{} from \citet{yang2022faster} and for \SGDASim{} in \cref{thm:sim-upper} yield a tight complexity:
\begin{equation}
 N_{\eps,\mathrm{Sim}}^\star, N_{\eps,\mathrm{Alt}}^\star
 \asymp \frac{\kappa^2\ell\Delta}{\eps^2}+\frac{\kappa^4\ell\Delta\sigma^2}{\eps^4}.
 \label{eq:tight-complexity}
\end{equation}
\paragraph{Suboptimality of SGDA.}
\cref{thm:common-lower} establishes a provable suboptimality of SGDA in NC-P{\L} games: it yields a tight sample complexity $\Theta(\kappa^2\ell\Delta\eps^{-2}+\kappa^4\ell\Delta\sigma^2\eps^{-4})$ for SGDA, which is suboptimal compared with the complexity upper bound $O(\kappa\ell\Delta\eps^{-2}+\kappa^2\ell\Delta\sigma^2\eps^{-4})$ of Smoothed-AGDA \citep{yang2022faster}.\footnote{The Smoothed-AGDA bound uses the comparison conditions and output conversion in Appendix~\ref{app:smoothing-alignment}.} Hence, our result justifies the recent development of alternative methods such as Smoothed-AGDA \citep{yang2022faster} and momentum \citep{huang2025enhanced,cai2026bias}.

\paragraph{Optimal timescale ratio and step-size schedule.}
Notably, every existing analysis of two-timescale SGDA in NC-P{\L} games relies on the choice of timescale ratio $\rho=\Theta(\kappa^2)$ \citep{fiez2021global,yang2022faster,cho2023sgda,ha2026high}. However, the question whether it is actually optimal remained open. Together with the matching upper bounds \citep{yang2022faster,ha2026high}, \cref{thm:common-lower} shows that the timescale ratio $\rho=\Theta(\kappa^2)$ and the constant step size $\eta_t\equiv c_\eta\min\{1/\ell,\sqrt{\Delta/(\ell\sigma^2T)}\}$ are optimal for two-timescale SGDA with a fixed timescale ratio and non-increasing step sizes.

In the following sections, we study two distinctive regimes of timescale ratios, namely the high-ratio regime $\rho=\Omega(\kappa^2)$ and low-ratio regime $\rho=O(\kappa^2)$, and explain the intuition behind the hard instances that establish \cref{thm:common-lower}. We defer the proof of Theorem~\ref{thm:common-lower} to Appendix~\ref{app:common-lower-proof}.

%% file: sections/high_ratio.tex
\section{The high-ratio lower bound}
\label{sec:high-ratio}

In this section, we establish the lower bound for large timescale ratios $\rho=\Omega(\kappa^2)$. In this high-ratio regime, smaller descent steps slow down the outer progress, whereas larger ascent steps increase noise-induced errors that accumulate on the outer variable. These complementary obstructions prevent SGDA from obtaining sharper convergence rates.

\Needspace*{7\baselineskip}
\begin{restatable}[High-ratio lower bound]{theorem}{highRatioTheoremStatement}
\label{thm:high-ratio}
There are universal constants $c,c_{\rm high}>0$ and $\kappa_{\rm high}\ge2$ such that, for every choice of problem parameters $\ell,\Delta>0$, $\sigma\ge0$, and $\kappa\ge\kappa_{\rm high}$, time horizon $T\in\mathbb N_{\ge1}$, and timescale ratio $\rho\ge c_{\rm high}\kappa^2$, both \SGDASim{} and \SGDAAlt{} satisfy
\begin{equation}
 \mathcal R_{T,\mathrm{Sim}}(\rho), \mathcal R_{T,\mathrm{Alt}}(\rho)
 \ge c\min\!\left\{\ell\Delta,
 \max\!\left\{\frac{\rho \ell\Delta}{T},
 \sqrt{\frac{\kappa^4\ell\Delta\sigma^2}{T}}\right\}\right\}.
 \label{eq:high-ratio-result}
\end{equation}
\end{restatable}

Together with the matching upper bounds from \citet{yang2022faster} and Theorem~\ref{thm:sim-upper}, Theorem~\ref{thm:high-ratio} shows that the smallest possible timescale ratio $\rho=\Theta(\kappa^2)$ is optimal in the high-ratio regime, as larger ratios only increase the lower bound. This suggests that the standard choice $\rho=\Theta(\kappa^2)$ \citep{yang2022faster} might be optimal among all fixed ratios. We establish this optimality in Section~\ref{sec:low-ratio} by ruling out a uniform convergence guarantee for timescale ratios of smaller order $\rho=o(\kappa^2)$.

In the following subsections, we construct two hard NC-P{\L} game instances that yield Theorem~\ref{thm:high-ratio}.

\subsection{The instance for unstable step sizes}
\label{sec:high-ratio-large-steps}

Consider a simple linear quadratic game $f_{\mathrm{unstable}}:\R\times\R\to\R$ and its outer objective:
\begin{equation}
 f_{\mathrm{unstable}}(x;y)=\frac{\ell+\mu}{4}(xy-y^2),
 \qquad \Phi_{\mathrm{unstable}}(x)=\frac{\ell+\mu}{16}x^2.
 \label{eq:quadratic-main}
\end{equation}
This game is $\ell$-smooth and $\mu$-strongly concave in $y$, and therefore, is an NC-P{\L} game. For this game, an inner ascent step becomes unstable for $\eta_t\ge8/\ell$. The following result shows that such $m$ initial unstable steps can push $x$ away from its minimizer.

\begin{restatable}[A lower bound for $f_{\mathrm{unstable}}$]{lemma}{quadraticGradientLowerStatement}
\label{lem:quadratic-gradient-lower}
There is a universal $c>0$ such that, for $\ell,\Delta>0$, $\kappa\ge2$, $\sigma\ge0$, $\rho\ge1$, integers $T\ge1$ and $0\le m<T$, and predetermined nonnegative non-increasing step sizes with $\eta_t\ge8/\ell$ for $t<m$, each of \SGDASim{} and \SGDAAlt{} admits an initialization $(x_0,y_0)$ making $f_{\mathrm{unstable}}$ with $\mu=\ell/\kappa$ an exact-gradient instance in $\Class(\ell,\kappa,\Delta,\sigma)$ with
\begin{equation}
 \frac1T\sum_{t=0}^{T-1}|\Phi_{\mathrm{unstable}}'(x_t)|^2
 \ge c\frac{m+1}{T}\ell\Delta.
 \label{eq:quadratic-main-bound}
\end{equation}
\end{restatable}

For $m+1\ge T/2$, Lemma~\ref{lem:quadratic-gradient-lower} immediately proves Theorem~\ref{thm:high-ratio} with the constant lower bound $c\ell\Delta$. We construct a distinct hard instance $f_{\mathrm{stable}}$ for the other case $m+1<T/2$ in the next subsection.

\subsection{The instance for stable step sizes}
\label{sec:high-ratio-corridor-noise}

For the remaining case $m+1<T/2$, we design an instance $f_{\mathrm{stable}}$ with outer coordinates $x=(u,v)$:
\begin{equation}
 f_{\mathrm{stable}}(u,v;y)=g_{\mathrm{corridor}}(u)+g_\sigma(v,y),
 \qquad \Phi_{\mathrm{stable}}(u,v)=g_{\mathrm{corridor}}(u)+\Phi_\sigma(v),
 \label{eq:high-separable-game-main}
\end{equation}
where we use a scalar coordinate $u\in\R$ for the deterministic function $g_{\mathrm{corridor}}$, and define the outer objective $\Phi_\sigma(v):=\max_y g_\sigma(v,y)$ for the noisy function $g_\sigma(v,y)$. Intuitively, the corridor function $g_{\mathrm{corridor}}$ limits progress from small descent steps, and the noisy function $g_\sigma$ transfers the inner stochastic perturbations to the outer variable. The stable ascent steps sum to $S=\sum_{t=m}^{T-2}\eta_t$.

% Defer Figure 1 to the top of page 6, beside the stable-step construction.
\suppressfloats[t]
\input{figures/figure1}

\Needspace{4\baselineskip}
\subsubsection{The corridor function \texorpdfstring{$g_{\mathrm{corridor}}$}{g\_corridor}}
\label{sec:high-ratio-corridor}
The following lower bound shows that the corridor function $g_{\mathrm{corridor}}$ keeps the remaining primal gradients constant and nonzero.

\Needspace{7\baselineskip}
\begin{restatable}[A lower bound for $g_{\mathrm{corridor}}$]{lemma}{corridorFunctionLowerStatement}
\label{lem:one-dimensional-lower}
For $\ell,\Delta>0$, $\rho\ge1$, integers $T\ge1$, $0\le m<T$, and nonnegative steps with $S=\sum_{t=m}^{T-2}\eta_t$, there is an $\ell$-smooth function $g_{\mathrm{corridor}}:\R\to\R$ with initial gap $g_{\mathrm{corridor}}(0)-\inf_u g_{\mathrm{corridor}}(u)=\Delta$. Gradient descent from $u_0=0$ with steps $\eta_t/\rho$ satisfies $g_{\mathrm{corridor}}'(u_t)<0$ for all $t<T$ and $g_{\mathrm{corridor}}'(u_t)=-G$ for $m\le t<T$, where $G>0$ and $G^2\asymp\min\{\ell\Delta,\rho\Delta/S\}$. Consequentially,
\begin{equation}
 \frac1T\sum_{t<T}|g_{\mathrm{corridor}}'(u_t)|^2
 \gtrsim\frac{T-m}{T}\min\left\{\ell\Delta,\frac{\rho\Delta}{S}\right\}.
 \label{eq:one-dimensional-average-main}
\end{equation}
\end{restatable}

We illustrate the corridor function $g_{\mathrm{corridor}}$ in Figure~\ref{fig:tail-mechanisms} (a). Intuitively, it consists of an entry curve and a constant-slope segment, followed by an exit curve. The entry curve brings the first $m$ updates to the corridor, and the remaining iterates travel its length $GS/\rho$ at slope $-G$, using gap $G^2S/\rho$. The exit curve brings the slope to zero; with curvature of magnitude $\ell$, the entry and exit curves together use gap of order $G^2/\ell$. Hence $G^2(S/\rho+1/\ell)\asymp\Delta$ gives the stated lower bound (Appendix~\ref{app:corridor}).

While Lemma~\ref{lem:one-dimensional-lower} provides a lower bound for small cumulative steps $S$, it weakens as $S$ grows. We complement this by constructing the function $g_\sigma$ which induces a lower bound that grows with $S$.

\subsubsection{The lasting effect of inner perturbation}
\label{sec:high-ratio-prototype}
To construct a noisy function $g_\sigma$ whose lower bound grows with the cumulative steps $S$, we study an inner perturbation that causes a long-lasting outer error. To make our exposition clear, we represent the perturbation by a deterministic shift $\delta>0$ in a simple linear-quadratic game $g_\delta:\R\times\R^2\to\R$:
\begin{equation}
 g_\delta(z;a,b)=\ell za-\frac{\ell}{2\kappa}\bigl[a^2+(a+b+\delta)^2\bigr]
 \label{eq:local-perturbation-game-main}
\end{equation}
with an outer variable $z$ and inner variables $a,b$. Here, the inner maximizer $(\kappa z,-\kappa z-\delta)$ gives outer objective $\Phi_{g_\delta}(z)=\ell\kappa z^2/2$; hence, the inner perturbation $\delta$ does not affect the outer objective. From the origin, however, an ascent step follows the negative derivative $\partial_a g_\delta=-\ell\delta/\kappa$ and makes the inner variable $a$ negative. A descent step then follows the derivative $-\partial_zg_\delta=-\ell a>0$ and moves the outer variable $z$ away from zero, which is the minimizer of the outer objective $\Phi_{g_\delta}$. Consequentially, the inner perturbation $\delta$ induces an outer error through GDA dynamics.

In the high-ratio regime $\rho=\Omega(\kappa^2)$, such an outer error takes longer to remove than to create. To see this clearly, let us measure the time by cumulative steps $\sum_{i<t}\eta_i$. The outer variable $z$ moves slowly, so we may approximate its effect by fixing $z=0$. Then, the inner gradients follow $\partial_a g_\delta=-(\ell/\kappa)(2a+b+\delta)$ and $\partial_b g_\delta=-(\ell/\kappa)(a+b+\delta)$. In particular, the gradients start at $-\ell\delta/\kappa$ and remain negative while $|a|+|b|$ is small relative to $\delta$. Therefore, $a$ reaches a negative constant fraction of $\delta$ on the cumulative timescale $\kappa/\ell$; then it returns to zero as $b$ approaches $-\delta$. Substituting this response into $z_{i+1}-z_i=-\ell\eta_i a_i/\rho$ predicts the scale
\[
 \underbrace{\frac{\ell\delta}{\rho}}_{\text{outer speed}}
 \times\underbrace{\frac{\kappa}{\ell}}_{\text{formation time}}
 =\underbrace{\frac{\kappa\delta}{\rho}}_{\text{error size}}.
\]
After the inner variables adapt enough to the perturbation $\delta$, the inner maximizer gives the approximation $a\approx\kappa z$, so we obtain an outer contraction $z_{i+1}\approx(1-\ell\kappa\eta_i/\rho)z_i$. This suggests that removing a fixed fraction of the error induced by an inner perturbation takes cumulative steps of order $\rho/(\ell\kappa)$, a factor $\rho/\kappa^2$ longer than its formation time. The following lemma captures this intuition and bounds the error on the outer variable $z_t$, and specifies the interval over which this error persists.

\Needspace{8\baselineskip}
\begin{restatable}[Persistent error under a deterministic shift]{lemma}{shiftResponseStatement}
\label{lem:shift-response}
There are universal constants $c,C,C_0>0$ such that, for $\ell,\delta>0$, $\kappa\ge1$, $\rho\ge C_0\kappa^2$, and any finite step-size sequence with $0\le\ell\eta_t\le c\kappa$, both \SGDASim{} and \SGDAAlt{} with exact gradients on the linear quadratic game $g_\delta$ given by~\eqref{eq:local-perturbation-game-main}, initialized at $(0,0,0)$, satisfy $|z_t|\le C\kappa\delta/\rho$ and $|a_t|+|b_t|\le C\delta$ whenever $\sum_{i<t}\eta_i\le c\rho/(\ell\kappa)$, and
\begin{equation}
 z_t\ge c\frac{\kappa\delta}{\rho}
 \quad\text{if}\quad
 C\frac\kappa\ell\le\sum_{i<t}\eta_i\le c\frac\rho{\ell\kappa}.
 \label{eq:response-window-main}
\end{equation}
\end{restatable}

\subsubsection{The noisy function \texorpdfstring{$g_\sigma$}{g\_sigma}}
\label{sec:high-ratio-noisy}
We now construct the noisy function $g_\sigma$ with its coordinate-wise components $g_{\sigma,j}$, each of which implements $g_\delta$ through a stochastic noise oracle.
Specifically, we assign each update $j\in J:=\{0,\ldots,T-2\}$ a separate component $g_{\sigma,j}:\R\times\R^3\to\R$ with $y=(a,b,s)$, initialized at the origin in coordinates $v[j],y[j]=(a[j],b[j],s[j])$. We denote its outer objective by $\Phi_{\sigma,j}$, choose the constants $K\asymp\kappa$ and $L\asymp\ell$ with fixed universal factors, and set $g_{\sigma,j}\equiv0$ if $\eta_j\sigma=0$.
\begin{samepage}
\begin{align}
 g_\sigma(v,y)&=\sum_{j\in J}g_{\sigma,j}(v[j];y[j]),
 \label{eq:aggregate-noise-game}\\
 g_{\sigma,j}(v;a,b,s)&:=Lva-\frac{L}{2K}\left[a^2+\bigl(a+b+\chi_j(a+b)\delta_j(s)\bigr)^2\right],
 \label{eq:noise-component-main}\\
 \delta_j(s)&:=\sqrt K\,\eta_j\sigma\sin^2\!\left(\frac{\pi s}{2\eta_j\sigma}\right).
 \label{eq:noise-activation-main}
\end{align}
The smooth cutoff $\chi_j:\R\to[0,1]$ equals one along the trajectory segment used in Lemma~\ref{lem:shift-response}, preserving $g_\delta$ with parameters $L,K$ and inner perturbation $\delta_j(s)$. The cutoff function $\chi_j$ and the choice of constants $K,L$ ensure $\ell$-smoothness and the inner $(\ell/\kappa)$-P{\L} inequality (Appendix~\ref{app:one-perturbation-conversion}).

Our stochastic oracle induces a gradient error $\pm\sigma$ in $s[j]$, and SGDA turns that error into $s[j]=\pm\eta_j\sigma$.
Specifically, let $e(y)$ be the unit vector in the $s[j]$-coordinate for the smallest $j\in J$ with $s[j]=0$, or zero if none exists.
With fresh independent uniform signs $\xi\in\{-1,1\}$, we introduce a~stochastic~oracle
\begin{equation}
 G_x=\nabla_x f_{\mathrm{stable}},\qquad
 G_y=\nabla_y g_\sigma+\sigma\xi e(y).
 \label{eq:noise-oracle-main}
\end{equation}
\end{samepage}
Since the derivative of $\delta_j$ vanishes at $s=\pm\eta_j\sigma$, each $s[j]$ stays fixed once activated, while the remaining $s$-coordinates stay at the origin; only $s[j]$ receives noise. The oracle has variance $\sigma^2\norm{e(y)}^2\le\sigma^2$, and the estimators in~\eqref{eq:noise-oracle-main} are pointwise unbiased, so they satisfy Assumptions~\ref{ass:unbiasedness}--\ref{ass:bounded-variance}.

For $j\ge m$ with enough remaining cumulative steps to cover $\asymp\rho/(\ell\kappa)$, Lemma~\ref{lem:shift-response} and curvature $\Phi_{\sigma,j}''\asymp\ell\kappa$ give an outer gradient $\Phi_{\sigma,j}'(v_t[j])\gtrsim\ell\kappa^{5/2}\eta_j\sigma/\rho$. Subsequent steps are at most $\eta_j$ due to the non-increasing step size schedule, hence, the following lower bound holds:
\begin{equation}
 \sum_{t<T}\E|\Phi_{\sigma,j}'(v_t[j])|^2
 \gtrsim
 \underbrace{\left(\frac{\ell\kappa^{5/2}\eta_j\sigma}{\rho}\right)^2}_{\text{squared outer gradient}}\times
 \underbrace{\frac{\rho}{\ell\kappa\eta_j}}_{\text{number of iterates}}
 =\frac{\ell\kappa^4\sigma^2\eta_j}{\rho}.
 \label{eq:one-noise-contribution-main}
\end{equation}

The next lemma captures this derivation, and lower bounds the outer gradient accumulation from a single inner perturbation.
\begin{restatable}[Contribution of a single inner-gradient perturbation]{lemma}{onePerturbationStatement}
\label{lem:one-perturbation-main}
There are universal constants $c,C,c_{\rm step},c_{\rm high}>0$ and $\kappa_{\rm high}\ge2$ such that, for $\ell,\sigma>0$, $\kappa\ge\kappa_{\rm high}$, $\rho\ge\max\{1,c_{\rm high}\kappa^2\}$, an integer $T\ge2$, any predetermined step sizes $\eta_0\ge\cdots\ge\eta_{T-2}\ge0$ and index $0\le j\le T-2$ with $0<\ell\eta_j\le c_{\rm step}\kappa$ admit an $\ell$-smooth NC-$(\ell/\kappa)$-P{\L} component $g_{\sigma,j}:\R\times\R^3\to\R$. Its outer objective $\Phi_{\sigma,j}(v):=\max_{y\in\R^3}g_{\sigma,j}(v;y)$ is quadratic with curvature $\Phi_{\sigma,j}''\asymp\ell\kappa$ and initial gaps $\Phi_{\sigma,j}(0)-\min_v\Phi_{\sigma,j}(v)=\Phi_{\sigma,j}(0)-g_{\sigma,j}(0;0)=0$. Starting from the origin with one equiprobable $\pm\sigma$ error in the third inner gradient at update $j$ and exact gradients otherwise, both \SGDASim{} and \SGDAAlt{} generate outer iterates $v_t[j]$ satisfying
\begin{equation}
 \sum_{t<T}\E|\Phi_{\sigma,j}'(v_t[j])|^2
 \ge c\frac{\ell\kappa^4\sigma^2\eta_j}{\rho}
 \quad\text{if}\quad
 \sum_{t=j+1}^{T-2}\eta_t\ge C\frac{\rho}{\ell\kappa}.
 \label{eq:one-perturbation-gradient-bound}
\end{equation}
\end{restatable}

\subsubsection{Combining the corridor function \texorpdfstring{$g_{\mathrm{corridor}}$}{g\_corridor} and noisy function \texorpdfstring{$g_\sigma$}{g\_sigma}}
\label{sec:high-ratio-combination}

With enough remaining cumulative steps, $S\ge C\rho/(\ell\kappa)$, we can sum Lemma~\ref{lem:one-perturbation-main}, omitting late perturbations with insufficient subsequent cumulative steps; their total step size is $O(\rho/(\ell\kappa))$. The corridor covers the other case of small remaining cumulative steps. The following lemma combines these bounds and captures a tradeoff in $S$ (Appendix~\ref{app:tail-lower}).
\Needspace*{11\baselineskip}

\begin{restatable}[Stable-step lower bound and admissibility]{lemma}{tailLowerStatement}
\label{lem:tail-lower}
There are universal $c,c_{\rm high}>0$ and $\kappa_{\rm high}\ge2$ such that, for $\ell,\Delta>0$, $\sigma\ge0$, $\kappa\ge\kappa_{\rm high}$, $\rho\ge\max\{1,c_{\rm high}\kappa^2\}$, an integer $T\ge1$, and predetermined steps $\eta_0\ge\cdots\ge\eta_{T-2}\ge0$, the following holds. Let $m$ count the entries $\eta_t\ge8/\ell$ and assume $m+1<T/2$. For each of \SGDASim{} and \SGDAAlt{}, the game $f_{\mathrm{stable}}$ given by~\eqref{eq:high-separable-game-main} and a fixed oracle can be chosen as an instance in $\Class(\ell,\kappa,\Delta,\sigma)$ with initial outer gap $\Delta$ and inner gap zero at the origin. Writing $S=\sum_{t=m}^{T-2}\eta_t$ and $\rho\Delta/0:=+\infty$, its iterates satisfy
\begin{equation}
 \frac1T\sum_{t=0}^{T-1}\E\norm{\nabla\Phi_{\mathrm{stable}}(u_t,v_t)}^2
 \ge c\min\!\left\{\ell\Delta,
 \max\!\left\{\frac{\rho\Delta}{S},\frac{\kappa^4\ell\sigma^2S}{\rho T}\right\}\right\}.
 \label{eq:tail-lower-main}
\end{equation}
\end{restatable}

Figure~\ref{fig:tail-mechanisms} illustrates the tradeoff suggested by Lemma~\ref{lem:tail-lower}: increasing $S$ weakens the corridor bound $\rho\Delta/S$ but strengthens the stochastic bound $\kappa^4\ell\sigma^2S/(\rho T)$. Applying the AM--GM inequality and $S<8T/\ell$ to Lemma~\ref{lem:tail-lower} proves Theorem~\ref{thm:high-ratio} as follows:
\begin{equation}
 \max\!\left\{\frac{\rho\Delta}{S},\frac{\ell\kappa^4\sigma^2S}{\rho T}\right\}
 \ge\max\!\left\{\frac{\rho\ell\Delta}{8T},\sqrt{\frac{\ell\kappa^4\Delta\sigma^2}{T}}\right\}.
 \label{eq:high-ratio-balance-main}
\end{equation}

%% file: figures/figure1.tex
% The v3 panel PDFs have a common vertical crop; equal heights preserve alignment.
\begin{figure}[!t]
  \centering
  % Remove subfig's default 10pt padding above this row.
  \captionsetup[subfloat]{farskip=0pt,justification=centering}
  \subfloat[The corridor function\label{fig:tail-corridor}]{%
    \includegraphics[height=1.70in]{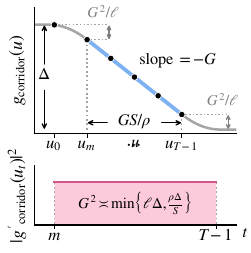}}\hfill
  \subfloat[Inner suboptimality\label{fig:tail-inner}]{%
    \includegraphics[height=1.70in]{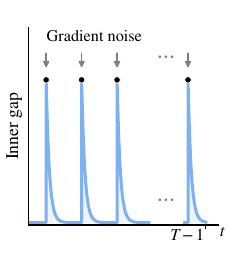}}\hfill
  \subfloat[\mbox{Mean squared outer gradient}\label{fig:tail-outer}]{%
    \makebox[0.34\linewidth][c]{\includegraphics[height=1.70in]{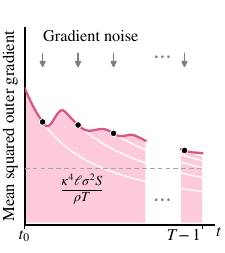}}}
  \caption{\textbf{The competing bounds in $f_{\mathrm{stable}}$ (schematic).}
  Let $m$ count the initial large steps and $S=\sum_{t=m}^{T-2}\eta_t$.
  (a) The corridor keeps gradient magnitude $G$ at $u_m,\ldots,u_{T-1}$.
  (b) Gradient noise creates inner gap $\Phi_\sigma(v_t)-g_\sigma(v_t,y_t)$.
  (c) Pink bands show contributions to $t^{-1}\sum_{i<t}\|\nabla\Phi_\sigma(v_i)\|^2$; the initial increases are omitted. The dashed level is the bound's scale $\kappa^4\ell\sigma^2S/(\rho T)$ for $S\ge C\rho/(\ell\kappa)$.}
  \label{fig:tail-mechanisms}
\end{figure}

%% file: sections/low_ratio.tex
\Needspace{18\baselineskip}
\section{Low-ratio lower bound}
\label{sec:low-ratio}

In the low ratio regime $\rho=O(\kappa^2)$, inner variables may fail to track their maximizer due to relatively small inner step sizes. In particular, \citet{mourrat2026pl} recently constructed a P{\L}-P{\L} game that exhibits a periodic orbit of a continuous-time GDA flow, suggesting that a P{\L} inequality can fail to prevent limit cycles of GDA. Below, we realize a similar construction for discrete-time GDA in NC-P{\L}~games.

\subsection{Inner P{\L} does not prevent a limit cycle}
\label{sec:low-ratio-small-steps}

We start with a simple bilinear game that creates a periodic orbit for GDA:
\begin{equation}
 g(u,a)=\sqrt2\,u^\trans a,\qquad u,a\in\R^2.
 \label{eq:low-bilinear-main}
\end{equation}
For timescale ratio $\rho=1$, GDA flow oscillates in this game \citep{hsieh2021limits}, but the infinite inner maximum for $u\ne0$ violates NC-P{\L}. Adding $-\norm a^2$ gives the outer objective $\norm u^2/2$ yet damps the flow: $\ddot u+2\dot u+2u=0$. Hence, we introduce $b\in\R^2$ and a residual $R:\R^2\times\R^2\to\R^2$ and define
\begin{equation}
 g_{\mathrm{cycle}}(u,a,b)=\sqrt2\,u^\trans a-\norm a^2-\tfrac12\norm{R(a,b)}^2.
 \label{eq:low-game-main}
\end{equation}

\paragraph{Designing the residual $R(a,b)$.}
We want the residual $R(a,b)$ to restore a periodic motion while preserving an inner P{\L} inequality and the quadratic outer objective $\norm u^2/2$. If every $a$ admits $b$ that satisfies $R(a,b)=0$, maximization over $(a,b)$ can cancel the residual and leave the outer objective unchanged. On the other hand, the inner P{\L} inequality follows when the map $(a,b)\mapsto(\sqrt2a,R(a,b))$ has a global smooth inverse with a uniformly bounded derivative (c.f.\ Lemma~\ref{lem:completion-square}). To specify the periodic motion we want $g_{\mathrm{cycle}}$ to induce, let $J$ denote counterclockwise rotation by $\pi/2$, and define
\begin{equation}
 u_\star(t)=(\cos t,\sin t)^\trans,\qquad
 a_\star(t)=-Ju_\star(t)/\sqrt2,\qquad
 b_\star(t)=u_\star(t)/\sqrt2.
 \label{eq:low-cycle-main}
\end{equation}
Our goal is to design the residual $R(a,b)$ for GDA to follow this path. Outer descent already has the required velocity along this path, since $\dot u=Ju=-\sqrt2a$. The inner variables must have velocities $\dot a=Ja$ and $\dot b=-a$. Hence, to obtain the desired inner motion, we can impose the restriction
\begin{equation}
 \nabla_a\!\left(\tfrac12\norm{R(a,b)}^2\right)=-2a+Ja,\qquad
 \nabla_b\!\left(\tfrac12\norm{R(a,b)}^2\right)=a,
 \label{eq:low-orbit-conditions-main}
\end{equation}
on the circle $b=Ja$, $\norm a=1/\sqrt2$ traced by the motion $(a_\star,b_\star)$. With these identities, the GDA equations agree with the target path. We choose further derivatives so that nearby trajectories approach the orbit, and use cutoffs for global smoothness without changing these dynamics (c.f.\ Appendix~\ref{app:low-ratio-periodic}).

\begin{restatable}[An attracting periodic orbit of GDA flow in $g_{\mathrm{cycle}}$]{lemma}{lowPeriodicGameStatement}
\label{lem:low-periodic-solution}
There are universal $C,\mu,h_\star>0$ and a smooth residual $R$ such that the base game $g_{\mathrm{cycle}}:\R^2\times\R^4\to\R$ in \eqref{eq:low-game-main} is jointly $C$-smooth and affine in $u$, has outer objective $\Phi_{g_{\mathrm{cycle}}}(u)=\tfrac12\norm u^2$, and satisfies the inner $\mu$-P{\L} inequality.
The path~\eqref{eq:low-cycle-main}, with $y_\star=(a_\star,b_\star)$, traces a locally stable $2\pi$-periodic orbit $\Gamma$ of GDA flow. There is a compact neighborhood $\mathcal U$ of $\Gamma$ such that both \SGDASim{} and \SGDAAlt{}, started in $\mathcal U$ with exact gradients and equal primal and dual steps $h_t\in[0,h_\star]$, satisfy $(u_t,y_t)\in\mathcal U$ and $\norm{\nabla\Phi_{g_{\mathrm{cycle}}}(u_t)}=\norm{u_t}\ge\tfrac12$ for all $t\ge0$, where $y_t=(a_t,b_t)$.
\end{restatable}

\subsection{From the cycle to the low-ratio lower bound}
\label{sec:low-ratio-initial-steps}

Lemma~\ref{lem:low-periodic-solution} shows that continuous-time GDA can be trapped in a limit cycle in $g_{\mathrm{cycle}}$. However, its discrete-time guarantee remains local and requires step sizes at most $h_\star$, which can be extremely small. The next lemma fills this gap by modifying the game, and shows that GDA can reach the orbit from a suitable initialization after any finite sequence of larger steps and remain nearby under subsequent small steps.

\begin{restatable}[GDA can reach the orbit in a modified game $\tilde g_{\mathrm{cycle}}$]{lemma}{lowInitialUpdatesStatement}
\label{lem:low-initial-updates}
There are universal $C,\mu>0$ with the following property. For any finite sequence of steps $h_0,\ldots,h_{m-1}>h_\star$, including $m=0$, we can modify the game $g_{\mathrm{cycle}}$ in \cref{lem:low-periodic-solution} to obtain $\tilde{g}_{\mathrm{cycle}}:\R^2\times\R^5\to\R$. This game is jointly $C$-smooth and affine in $u$, satisfies the inner $\mu$-P{\L} inequality, and has outer objective $\Phi_{\tilde{g}_{\mathrm{cycle}}}(u)=\norm u^2/2$. With $y=(a,b,r)$ and $e_1=(1,0)^\trans$, there is an initialization $(e_1,y_0)$ with inner gap at most $C$. From this initialization, both GDA\textsubscript{Sim} and GDA\textsubscript{Alt}, using exact gradients and equal primal and dual steps, keep $u_t=e_1$ for $0\le t\le m$ and reach $y_m=(y_\star(0),0)$.
For every continuation with $h_t\in[0,h_\star]$, each algorithm follows its $g_{\mathrm{cycle}}$ dynamics in $(u,a,b)$, while $r$ remains zero.
\end{restatable}

\Needspace{19\baselineskip}
We now rescale $\tilde g_{\mathrm{cycle}}$ to satisfy $\ell$-smoothness and the inner $(\ell/\kappa)$-P{\L} inequality. We choose $f_{\mathrm{cycle}}(x,y)=A\tilde g_{\mathrm{cycle}}(x/s_x,y/s_y)$ with positive scales satisfying $s_y^2=\rho s_x^2$, which converts both updates at ratio $\rho$ to equal-step dynamics. The scales in Appendix~\ref{app:low-ratio-theorem} make $f_{\mathrm{cycle}}$ $\ell$-smooth with inner P{\L} constant of order $\ell/\sqrt\rho$, so $\rho\le c_{\rm low}\kappa^2$ ensures inner P{\L} at least $\ell/\kappa$. They also meet the initial-gap budgets and give $\norm{\nabla\Phi_{\mathrm{cycle}}(x_t)}^2\gtrsim\ell\Delta$ at every iterate, proving the following theorem.

\begin{restatable}[Necessity of an order-$\kappa^2$ fixed ratio]{theorem}{lowRatioTheoremStatement}
\label{thm:low-ratio}
There are universal $c,c_{\rm low}>0$ and $\kappa_{\rm low}\ge2$ such that, for every $\kappa\ge\kappa_{\rm low}$, $\ell,\Delta>0$, $\sigma\ge0$, integer $T\ge1$, $1\le\rho\le c_{\rm low}\kappa^2$, and predetermined nonnegative non-increasing step sizes, each of \SGDASim{} and \SGDAAlt{} admits a deterministic instance $f_{\mathrm{cycle}}$ in $\Class(\ell,\kappa,\Delta,\sigma)$ with $d_x=2$, $d_y=5$, a strongly convex quadratic outer objective, and
\begin{equation}
 \min_{0\le t<T}\norm{\nabla\Phi(x_t)}^2\ge c\ell\Delta.
 \label{eq:low-ratio-result}
\end{equation}
\end{restatable}

Crucially, the theorem answers the question raised by \citet{cho2023sgda} affirmatively: \emph{timescale ratio $\rho=\Omega(\kappa^2)$ is necessary for SGDA to converge in NC-P{\L} games.} In particular, it shows that smaller timescale ratios $\rho=o(\kappa^2)$ cannot provide convergence guarantees in general, even with deterministic gradients and non-increasing step sizes. Together with Theorem~\ref{thm:high-ratio}, it yields our tight main lower bound in Theorem~\ref{thm:common-lower}.

%% file: sections/conclusion.tex
\section{Conclusion}
\label{sec:conclusion}

In this work, we established the first tight lower bounds of two-timescale SGDA with a fixed timescale ratio and nonincreasing step sizes in NC-P{\L} games. The lower bounds explain the fundamental limitation of two-timescale SGDA: small ratios can keep outer gradients bounded away from zero, while large ratios leave a tradeoff between outer progress and errors induced by inner noise. Our negative result justifies the recent line of work developing alternative methods, such as Smoothed-AGDA \citep{yang2022faster} and methods based on momentum \citep{huang2025enhanced,cai2026bias}.

On the other hand, our analysis only applies to SGDA with a fixed timescale ratio and nonincreasing step sizes; therefore, it does not rule out nonmonotone or signed step size schedules \citep{altschuler2025silver,shugart2025negative} that might improve SGDA's complexity. An additional interesting direction to explore would be to extend the ideas behind our lower bounds to nonconvex-concave problems. We leave these two questions as intriguing future work.

%% file: appendices/main_result.tex
\section{Proof of main result}
\label{app:main-result}

We first combine the two ratio regimes to prove Theorem~\ref{thm:common-lower}. We then prove Theorems~\ref{thm:high-ratio} and~\ref{thm:low-ratio} using the construction lemmas from Sections~\ref{sec:high-ratio} and~\ref{sec:low-ratio}, and establish the matching upper bound for \SGDASim{}.

\subsection{Proof of Theorem~\ref{thm:common-lower}}
\label{app:common-lower-proof}

\begingroup
\renewcommand{\theHequation}{restated.common-theorem.\arabic{equation}}
\commonLowerStatement*
\endgroup

\begin{proof}
We first choose the construction constants so that the two ratio ranges overlap. The fixed cycle and its extension in Lemma~\ref{lem:low-initial-updates} determine a universal inner-P{\L} constant $\mu$. Let $C_{\mathrm{low}}$ be the smoothness-normalization constant chosen in Appendix~\ref{app:low-ratio-theorem}; the low-ratio construction then applies up to $c_{\mathrm{low}}=(\mu/C_{\mathrm{low}})^2$. These constants are fixed before we calibrate the noise component. Let $C_{\rm resp}$ be the ratio constant $C_0$ in \cref{lem:shift-response}, enlarged as in Appendix~\ref{app:one-perturbation-conversion}, let $c_0,C_0$ be the constants from \cref{lem:raw-perturbation-game}, and choose
\[
 0<\gamma\le
 \min\left\{\sqrt{\frac{c_{\mathrm{low}}}{C_{\rm resp}}},\frac{c_0}{2C_0}\right\},
 \qquad \theta:=\frac{\gamma}{c_0}.
\]
These choices calibrate only the high-ratio block: they give \eqref{eq:high-ratio-scale-choice} and $c_{\rm high}=\gamma^2C_{\rm resp}\le c_{\mathrm{low}}$. Thus its ratio range begins no later than the endpoint of the low-ratio range. The low-ratio constants do not depend on $\gamma$ or $\theta$.

Choose $\kappa_0$ above both theorem thresholds. In the high-ratio bound, $\rho\ell\Delta/T\ge c_{\rm high}\kappa^2\ell\Delta/T$; the stochastic term already matches \eqref{eq:common-lower}. The low-ratio bound supplies the cap $c\ell\Delta$. Since the ranges cover every $\rho\ge1$, taking the ratio infimum proves \eqref{eq:common-lower}.

For $\eps^2\le c\ell\Delta$ with sufficiently small universal $c>0$, inversion gives the lower complexity bound in \eqref{eq:tight-complexity}. The matching constant-step upper bounds cited in Section~\ref{sec:main-results} complete the proof.
\end{proof}

\subsection{Proof of Theorem~\ref{thm:high-ratio}}
\label{app:high-ratio-theorem}

\begingroup
\renewcommand{\theHequation}{restated.high-theorem.\arabic{equation}}
\highRatioTheoremStatement*
\endgroup

\begin{proof}
For each of \SGDASim{} and \SGDAAlt{} with predetermined non-increasing step sizes, let $m$ count the initial excessively large steps $\ell\eta_t\ge8$ and set $S=\sum_{t=m}^{T-2}\eta_t$. If $m+1\ge T/2$, Lemma~\ref{lem:quadratic-gradient-lower} gives an average squared outer gradient at least $c\ell\Delta$ on $f_{\mathrm{unstable}}$, proving the capped bound. Otherwise, Lemma~\ref{lem:tail-lower} applies to $f_{\mathrm{stable}}$ and also gives the cap when $S=0$. For $S>0$, the remaining steps satisfy $S\le8T/\ell$, while the maximum of two nonnegative numbers is at least their geometric mean. Hence
\begin{equation}
 \max\!\left\{\frac{\rho\Delta}{S},
 \frac{\kappa^4\ell\sigma^2S}{\rho T}\right\}
 \ge\max\!\left\{\frac{\rho\ell\Delta}{8T},
 \sqrt{\frac{\kappa^4\ell\Delta\sigma^2}{T}}\right\}.
 \label{eq:high-ratio-cancellation}
\end{equation}
Substituting into \eqref{eq:tail-lower-main} and adjusting the universal constant proves \eqref{eq:high-ratio-result}. In each case, the instance is chosen after fixing the schedule. Taking the instance supremum and then the schedule infimum proves the theorem.
\end{proof}

\subsection{Proof of Theorem~\ref{thm:low-ratio}}
\label{app:low-ratio-theorem}
\label{app:low-ratio-scaling}

\begingroup
\renewcommand{\theHequation}{restated.low-theorem.\arabic{equation}}
\lowRatioTheoremStatement*
\endgroup

\begin{proof}
We rescale the fixed-dimensional game while retaining its strongly convex quadratic outer objective, then verify smoothness, inner P{\L}, and the initial-gap budgets.

Let $C\ge1$ bound the smoothness constant and initial inner gap in \cref{lem:low-initial-updates}. Its construction also gives $\norm{\nabla_u\tilde{g}_{\mathrm{cycle}}(e_1,y_0)}\le1$, as verified in Appendix~\ref{app:low-ratio-interpolation}. Choose a universal $C_{\mathrm{low}}\ge2C$, and fix $0<\zeta\le\min\{1/2,c_{\rm init}/(3C)\}$. This choice also covers the regularized initial gap in Appendix~\ref{app:smoothed-initialization}. For a ratio $\rho$, set
\begin{equation}
 A:=\frac{2\zeta\Delta}{\sqrt \rho},
 \qquad
 s_x^2:=\frac{2C_{\mathrm{low}}\zeta\Delta}{\ell\rho},
 \qquad
 s_y^2:=\frac{2C_{\mathrm{low}}\zeta\Delta}{\ell}.
 \label{eq:low-exact-scales-app}
\end{equation}
For a fixed horizon and predetermined non-increasing step sizes $(\eta_t)$, put $h_t=\ell\eta_t/(C_{\mathrm{low}}\sqrt \rho)$. Let $m$ count the initial steps with $h_t>h_\star$. Apply \cref{lem:low-initial-updates} to those steps, writing $\tilde{g}_{\mathrm{cycle}}$ for its game. For both \SGDASim{} and \SGDAAlt{}, the lemma keeps $u_t=e_1$ through time $m$ and then identifies the continuation with the small-step dynamics of $g_{\mathrm{cycle}}$. Hence \cref{lem:low-periodic-solution} keeps all later states in $\mathcal U\times\{0\}$, so $\norm{u_t}\ge1/2$ at every iterate. This includes $m=0$, $m=T-1$, $T=1$, and zero later steps.

Define the game $f_{\mathrm{cycle}}:\R^2\times\R^5\to\R$ by
\begin{equation}
 f_{\mathrm{cycle}}(x,y):=A\,\tilde{g}_{\mathrm{cycle}}\left(\frac{x}{s_x},\frac{y}{s_y}\right),
 \label{eq:low-rescaled-game-app}
\end{equation}
initialized at $x_0=s_xe_1$ and $y_0^{\mathrm{cycle}}=s_y y_0$. Its outer objective and Hessian are
\begin{equation}
 \Phi_{\mathrm{cycle}}(x)=\frac{A}{2s_x^2}\norm x^2
 =\frac{\ell\sqrt \rho}{2C_{\mathrm{low}}}\norm x^2,
 \qquad
 \nabla^2\Phi_{\mathrm{cycle}}=\frac{\ell\sqrt \rho}{C_{\mathrm{low}}}I_2\succ0.
 \label{eq:low-quadratic-value-app}
\end{equation}
Thus the full outer objective, not merely one component, is strongly convex and quadratic.

After substituting $u=x/s_x$ and $v=y/s_y$ into the updates, both gradient coefficients equal
\begin{equation}
 \frac{\eta_t A}{\rho s_x^2}=\frac{\eta_t A}{s_y^2}
 =\frac{\ell\eta_t}{C_{\mathrm{low}}\sqrt \rho}=h_t.
 \label{eq:low-induced-step-app}
\end{equation}
Since $f_{\mathrm{cycle}}$ is affine in $x$, its Hessian with respect to $x$ vanishes: $\nabla^2_{xx}f_{\mathrm{cycle}}=0$. For every $(x,y)$, differentiating the partial gradients $\nabla_xf_{\mathrm{cycle}}$ and $\nabla_yf_{\mathrm{cycle}}$ with respect to $y$ gives
\begin{align*}
 \norm{\nabla^2_{xy}f_{\mathrm{cycle}}(x,y)}_{\mathrm{op}}
 &\le C\frac{A}{s_xs_y}=C\frac{\ell}{C_{\mathrm{low}}},\\
 \norm{\nabla^2_{yy}f_{\mathrm{cycle}}(x,y)}_{\mathrm{op}}
 &\le C\frac{A}{s_y^2}=C\frac{\ell}{C_{\mathrm{low}}\sqrt \rho}.
\end{align*}
The operator norm of the full Hessian is at most the sum of these two bounds, which is at most $2C\ell/C_{\mathrm{low}}\le \ell$ since $\rho\ge1$. Thus the full gradient is $\ell$-Lipschitz. If $\mu$ is the inner-P{\L} constant of $\tilde{g}_{\mathrm{cycle}}$, the scaled constant is $\mu A/s_y^2=\mu \ell/(C_{\mathrm{low}}\sqrt \rho)$, which is at least $\ell/\kappa$ when
\begin{equation}
 \rho\le c_{\mathrm{low}}\kappa^2,
 \qquad c_{\mathrm{low}}:=\left(\frac{\mu}{C_{\mathrm{low}}}\right)^2.
 \label{eq:low-ratio-threshold-app}
\end{equation}
By \cref{lem:low-initial-updates}, $\tilde{g}_{\mathrm{cycle}}(x/s_x,\cdot)$ has a maximizer for every fixed $x$. Multiplying that maximizer by $s_y$ gives a maximizer of $f_{\mathrm{cycle}}(x,\cdot)$, since $A>0$.

The initial outer gap $\Phi_{\mathrm{cycle}}(x_0)-\inf_x\Phi_{\mathrm{cycle}}(x)$ and inner gap $\Phi_{\mathrm{cycle}}(x_0)-f_{\mathrm{cycle}}(x_0,y_0^{\mathrm{cycle}})$ satisfy the budgets in \cref{ass:initialization}:
\begin{align}
 \Phi_{\mathrm{cycle}}(x_0)-\inf_x\Phi_{\mathrm{cycle}}(x)
 &=\frac A2=\frac{\zeta\Delta}{\sqrt \rho}\le\Delta,\notag\\
 \Phi_{\mathrm{cycle}}(x_0)-f_{\mathrm{cycle}}(x_0,y_0^{\mathrm{cycle}})
 &=A\left[\frac12-\tilde{g}_{\mathrm{cycle}}(e_1,y_0)\right]
 \le CA\le2C\zeta\Delta\le c_{\rm init}\Delta.
 \label{eq:low-gap-budgets-app}
\end{align}
The last inequalities use $\rho\ge1$ and the chosen $\zeta$. Finally, every iterate satisfies
\begin{equation}
 \norm{\nabla\Phi_{\mathrm{cycle}}(x_t)}^2
 =\frac{A^2}{s_x^2}\norm{u_t}^2
 =\frac{2\zeta \ell\Delta}{C_{\mathrm{low}}}\norm{u_t}^2
 \ge\frac{\zeta}{2C_{\mathrm{low}}}\ell\Delta.
 \label{eq:low-gradient-app}
\end{equation}
The instance therefore belongs to $\Class(\ell,\kappa,\Delta,\sigma)$ for every $\sigma\ge0$, using exact gradients. Its dimension is $(d_x,d_y)=(2,5)$, independent of the horizon and schedule. Equation~\eqref{eq:low-gradient-app} proves the every-iterate bound \eqref{eq:low-ratio-result}. Averaging this inequality, taking the instance supremum, and then taking the schedule infimum in \eqref{eq:minimax-risks} gives $\mathcal R_{T,\mathsf o}(\rho)\ge c\ell\Delta$ for both update orders.

\end{proof}

\input{appendices/simultaneous_upper}

%% file: appendices/simultaneous_upper.tex
\subsection{Upper bound for \SGDASim{}}
\label{app:simultaneous-upper}

We prove the matching upper bound for \SGDASim{} under the same assumptions and average squared-gradient criterion as the lower bound. The deterministic one-sided-P{\L} counterpart is already established by \citet[Appendix~B.2]{chen2022faster}; the argument below includes unbiased stochastic gradients with bounded variance and our convention of averaging over $x_0,\ldots,x_{T-1}$. The main issue is that $\nabla_x f(x_t,y_t)$ differs from $\nabla\Phi(x_t)$. Inner P{\L} controls this discrepancy through the inner gradient by \cref{lem:inner-pl-consequences}; the potential below uses a faster inner step to absorb it.

\begin{theorem}[Upper bound for \SGDASim{}]
\label{thm:sim-upper}
There are constants $c_{\mathrm{step}},C>0$, depending at most on $c_{\mathrm{init}}$, such that the following holds for every $\kappa\ge2$, $\ell,\Delta>0$, $\sigma\ge0$, and integer $T\ge1$. On any instance in $\Class(\ell,\kappa,\Delta,\sigma)$, run \SGDASim{} with
\begin{equation}
 \rho=16\kappa^2,
 \qquad
 \eta_x:=c_{\mathrm{step}}
 \min\left\{\frac1{\kappa^2\ell},
 \sqrt{\frac{\Delta}{\kappa^4\ell\sigma^2T}}\right\},
 \qquad
 \eta_t\equiv\eta_y:=16\kappa^2\eta_x,
 \label{eq:sim-upper-steps}
\end{equation}
where the second term is $+\infty$ when $\sigma=0$. Then
\begin{equation}
 \frac1T\sum_{t=0}^{T-1}\E\norm{\nabla\Phi(x_t)}^2
 \le C\left(
 \frac{\kappa^2\ell\Delta}{T}
 +\sqrt{\frac{\kappa^4\ell\Delta\sigma^2}{T}}
 \right).
 \label{eq:sim-upper-risk}
\end{equation}
Consequently, whenever $\eps^2\le c\ell\Delta$ for a sufficiently small universal $c>0$,
\begin{equation}
 N_{\eps,\mathrm{Sim}}(16\kappa^2),
 N_{\eps,\mathrm{Sim}}^\star
 \le C\left(
 \frac{\kappa^2\ell\Delta}{\eps^2}
 +\frac{\kappa^4\ell\Delta\sigma^2}{\eps^4}
 \right).
 \label{eq:sim-upper-complexity}
\end{equation}
\end{theorem}

\begin{proof}[Proof of \cref{thm:sim-upper}]
We track the outer objective together with the inner suboptimality through
\begin{equation}
 \mathcal V_t:=\Phi(x_t)+\frac18[\Phi(x_t)-f(x_t,y_t)].
 \label{eq:sim-upper-potential}
\end{equation}
This potential is bounded below by $\Phi^*$. Its decrease combines $9/8$ times the outer-objective decrease with $1/8$ times the game-value increase. The faster dual step makes the latter absorb the discrepancy between the primal and outer gradients.

Let $\mathcal F_t$ be the history before the two oracle calls at iteration $t$ and $\E_t[\cdot]:=\E[\cdot\mid\mathcal F_t]$. The constant primal and dual steps are $\eta_x$ and $\eta_y=16\kappa^2\eta_x$ from \eqref{eq:sim-upper-steps}. Choose the universal constant $c_{\mathrm{step}}\le1/128$. Then
\begin{equation}
 \ell_\Phi\eta_x\le\frac14,
 \qquad
 \ell\eta_x\le\frac18,
 \qquad
 \ell\eta_y\le\frac18.
 \label{eq:sim-step-smallness}
\end{equation}

By $\ell_\Phi$-smoothness of $\Phi$ and the unbiased primal oracle,
\begin{equation}
 \begin{aligned}
 \E_t[\Phi(x_t)-\Phi(x_{t+1})]
 &\ge \eta_x\ip{\nabla\Phi(x_t)}{\nabla_x f(x_t,y_t)}
 -\frac{\ell_\Phi\eta_x^2}{2}(\norm{\nabla_x f(x_t,y_t)}^2+\sigma^2)\\
 &=\frac{\eta_x}{2}\norm{\nabla\Phi(x_t)}^2
 +\frac{\eta_x}{2}(1-\ell_\Phi\eta_x)\norm{\nabla_x f(x_t,y_t)}^2
 \\
 &\quad-\frac{\eta_x}{2}\norm{\nabla_x f(x_t,y_t)-\nabla\Phi(x_t)}^2
 -\frac{\ell_\Phi\eta_x^2}{2}\sigma^2.
 \end{aligned}
 \label{eq:sim-outer-descent}
\end{equation}
\Cref{ass:smoothness} implies that the full gradient is $2\ell$-Lipschitz in the Euclidean product norm. Since the primal and dual gradients in \SGDASim{} are evaluated at $(x_t,y_t)$, this gives
\begin{align}
 \E_t[f(x_{t+1},y_{t+1})-f(x_t,y_t)]
 &\ge-\left(\eta_x+\ell\eta_x^2\right)\norm{\nabla_x f(x_t,y_t)}^2
 +\left(\eta_y-\ell\eta_y^2\right)\norm{\nabla_y f(x_t,y_t)}^2\notag\\
 &\hspace{18mm}-\ell(\eta_x^2+\eta_y^2)\sigma^2.
 \label{eq:sim-game-increase}
\end{align}

Combining $\frac98$ times \eqref{eq:sim-outer-descent} with $\frac18$ times \eqref{eq:sim-game-increase}, using the gradient-bias bound \eqref{eq:primal-gradient-bias}, and then applying \eqref{eq:sim-step-smallness}, yields
\begin{equation}
 \E_t[\mathcal V_t-\mathcal V_{t+1}]
 \ge
 \frac{9\eta_x}{16}\norm{\nabla\Phi(x_t)}^2
 -33\kappa^4\ell\eta_x^2\sigma^2.
 \label{eq:sim-upper-one-step}
\end{equation}
Indeed, \eqref{eq:sim-step-smallness} makes the remaining squared-gradient coefficients nonnegative:
\begin{align*}
 \frac{9\eta_x}{16}(1-\ell_\Phi\eta_x)-\frac{\eta_x}{8}-\frac{\ell\eta_x^2}{8}&\ge0,\\
 \frac{\eta_y}{8}(1-\ell\eta_y)-\frac{9\eta_x\kappa^2}{16}&\ge0.
\end{align*}
The dual-noise contribution is $\ell\eta_y^2/8=32\kappa^4\ell\eta_x^2$, which determines the fourth power of $\kappa$ in the variance bound. The full noise coefficient satisfies
\[
 \frac{9\ell_\Phi\eta_x^2}{16}
 +\frac \ell8(\eta_x^2+\eta_y^2)
 \le33\kappa^4\ell\eta_x^2
 \qquad(\kappa\ge2).
\]

We sum \eqref{eq:sim-upper-one-step} over the $T-1$ updates, retaining the terminal potential to control the last reported iterate. Smoothness and lower boundedness of $\Phi$ give $\norm{\nabla\Phi(x)}^2\le2\ell_\Phi[\Phi(x)-\Phi^*]$. Thus \eqref{eq:sim-step-smallness} implies
\[
 \frac{9\eta_x}{16}\E\norm{\nabla\Phi(x_{T-1})}^2
 \le\frac9{32}\E[\mathcal V_{T-1}-\Phi^*].
\]
Adding this bound to the sum and using both initial-gap bounds gives
\begin{align}
 \frac{9\eta_x}{16}\sum_{t=0}^{T-1}\E\norm{\nabla\Phi(x_t)}^2
 &\le\mathcal V_0-\Phi^* -\frac{23}{32}\E[\mathcal V_{T-1}-\Phi^*]\notag\\
 &\quad+33\kappa^4\ell\eta_x^2\sigma^2(T-1)\notag\\
 &\le\left(1+\frac{c_{\mathrm{init}}}{8}\right)\Delta
 +33\kappa^4\ell\eta_x^2\sigma^2T.
 \label{eq:sim-upper-summed}
\end{align}
With the empty-sum convention, the same argument applies when $T=1$.
Dividing by $\eta_x T$ and absorbing the numerical constants gives
\begin{equation}
 \frac1T\sum_{t=0}^{T-1}\E\norm{\nabla\Phi(x_t)}^2
 \le C\left[\frac{\Delta}{\eta_x T}+\kappa^4\ell\eta_x\sigma^2\right].
 \label{eq:sim-upper-master}
\end{equation}

Substitute \eqref{eq:sim-upper-steps}. If the step-size bound $1/(\kappa^2\ell)$ is active, then $\sigma^2T\le \ell\Delta$ and both terms in \eqref{eq:sim-upper-master} are bounded by $C\kappa^2\ell\Delta/T$. Otherwise the two terms are both bounded by $C\sqrt{\kappa^4\ell\Delta\sigma^2/T}$. This proves \eqref{eq:sim-upper-risk}.

Finally, choosing
\[
 T\ge C\left(
 \frac{\kappa^2\ell\Delta}{\eps^2}
 +\frac{\kappa^4\ell\Delta\sigma^2}{\eps^4}
 \right)
\]
makes the right-hand side of \eqref{eq:sim-upper-risk} at most $\eps^2$. The constant schedule is nonincreasing, so the minimax definitions in \eqref{eq:minimax-risks} give \eqref{eq:sim-upper-complexity}.
\end{proof}

%% file: appendices/preliminaries.tex
\section{Preliminary lemmas}
\label{app:basic}

We first record the completion-of-squares construction used to preserve the outer function and inner P{\L}. We then prove the consequences of inner P{\L} used in the matching upper bound.

\subsection{NC-P{\L} games with a predetermined outer function}
\label{app:square-completion}

We use an onto inner map to determine the outer maximum exactly. A lower bound on its transposed derivative then converts the residual norm into an inner-gradient bound. This lemma proves inner P{\L}, not global smoothness; the latter is checked for each constructed map.

\begin{lemma}[Outer function and inner P{\L} for a completion-of-squares construction]
\label{lem:completion-square}
Let $b:\R^{d_x}\to\R^m$ and $T:\R^{d_y}\to\R^m$ be continuously differentiable, with $T$ onto (surjective). Define
\begin{equation}
 g(x,y):=\ip{b(x)}{T(y)}-\frac12\norm{T(y)}^2.
 \label{eq:completion-square-game}
\end{equation}
Then
\begin{equation}
 \max_y g(x,y)=\frac12\norm{b(x)}^2,
 \qquad
 \max_y g(x,y)-g(x,y)=\frac12\norm{b(x)-T(y)}^2.
 \label{eq:completion-square-gap}
\end{equation}
If, for some $\nu>0$,
\begin{equation}
 \norm{DT(y)^\trans z}\ge\nu\norm z
 \qquad\text{for all }y,z,
 \label{eq:completion-square-jacobian}
\end{equation}
then
\begin{equation}
 \frac12\norm{\nabla_y g(x,y)}^2
 \ge\nu^2\left[\max_v g(x,v)-g(x,y)\right].
 \label{eq:completion-square-pl}
\end{equation}
\end{lemma}

\begin{proof}
Since $T$ is onto, maximizing \eqref{eq:completion-square-game} over $y$ is equivalent to maximizing $\ip{b(x)}{z}-\norm z^2/2$ over $z\in\R^m$. Completing the square gives \eqref{eq:completion-square-gap}. Moreover,
\[
 \nabla_y g(x,y)=DT(y)^\trans[b(x)-T(y)],
\]
so \eqref{eq:completion-square-jacobian} and \eqref{eq:completion-square-gap} imply \eqref{eq:completion-square-pl}.
\end{proof}

\subsection{NC-P{\L} geometry}
\label{app:inner-pl-consequences}

The following are standard consequences of smoothness and inner P{\L}; see \citet[Lemma~A.5]{nouiehed2019solving}, \citet[Lemmas~A.7--A.8]{chen2022faster}, and \citet[Lemmas~A.2--A.3]{ha2026high}. We prove them under \cref{ass:smoothness,ass:nc-pl}, without concavity or uniqueness of the maximizer.

\begin{lemma}[Inner error bound, gradient bias, and outer smoothness]
\label{lem:inner-pl-consequences}
Under \cref{ass:smoothness,ass:nc-pl}, $\Phi$ is differentiable with $\nabla\Phi(x)=\nabla_x f(x,y^\star)$ for every $y^\star\in Y^\star(x)$. For every $(x,y)$,
\begin{align}
 \dist(y,Y^\star(x))
 &\le \frac{\kappa}{\ell}\norm{\nabla_yf(x,y)},
 \label{eq:inner-error-bound}\\
 \norm{\nabla_xf(x,y)-\nabla\Phi(x)}
 &\le \kappa\norm{\nabla_yf(x,y)}.
 \label{eq:primal-gradient-bias}
\end{align}
Moreover, $\Phi$ has $\ell_\Phi$-Lipschitz gradient with
\begin{equation}
 \ell_\Phi\le \ell(1+\kappa).
 \label{eq:outer-smoothness}
\end{equation}
\end{lemma}

\begin{proof}
We first bound the inner gap after changing $x$ while keeping a maximizer fixed; the resulting quadratic remainder proves differentiability of $\Phi$. An inner gradient-flow argument then bounds the distance to a maximizer, giving the gradient-bias and outer-smoothness estimates.

Fix $x$, $y^\star\in Y^\star(x)$, and $h\in\R^{d_x}$. Since the inner domain is unconstrained, $\nabla_y f(x,y^\star)=0$. Inner P{\L} at $(x+h,y^\star)$ and \cref{ass:smoothness} give
\[
 0\le\Phi(x+h)-f(x+h,y^\star)
 \le\frac{\norm{\nabla_y f(x+h,y^\star)}^2}{2\mu}
 \le\frac{\ell\kappa}{2}\norm h^2.
\]
Combining this with the Taylor bound for $f(\cdot,y^\star)$ yields
\[
 \left|\Phi(x+h)-\Phi(x)-\ip{\nabla_x f(x,y^\star)}h\right|
 \le\frac{\ell(1+\kappa)}2\norm h^2.
\]
Thus $\Phi$ is differentiable and $\nabla\Phi(x)=\nabla_x f(x,y^\star)$ for every maximizer, including when the maximizer is not unique. If $d_y=0$, this identity and all the claimed bounds follow directly from \cref{ass:smoothness}; henceforth assume $d_y\ge1$.

Fix $x$ and define the inner objective gap $g_x(y):=\Phi(x)-f(x,y)\ge0$. It is $\ell$-smooth, attains its minimum zero, and satisfies $\norm{\nabla g_x(y)}^2\ge2(\ell/\kappa)g_x(y)$. Since $\nabla g_x$ is globally Lipschitz, the gradient flow $\dot y(s)=-\nabla g_x(y(s))$, initialized at $y(0)=y$, exists for all $s\ge0$. If it reaches $g_x=0$ in finite time, the desired distance bound follows immediately by integrating up to that time. Otherwise $g_x(y(s))>0$ for every finite $s$, and
\[
 -\frac{d}{ds}\sqrt{g_x(y(s))}
 =\frac{\norm{\nabla g_x(y(s))}^2}{2\sqrt{g_x(y(s))}}
 \ge \sqrt{\frac{\ell}{2\kappa}}\norm{\nabla g_x(y(s))}.
\]
Hence the total distance traveled by the gradient-flow solution is finite:
\[
 \int_0^\infty\norm{\dot y(s)}\,ds
 \le \sqrt{\frac{2\kappa g_x(y)}{\ell}}.
\]
The finite path length makes $y(s)$ a Cauchy curve in $\R^{d_y}$, so it converges to a limit $y^\star$. The limit must satisfy $g_x(y^\star)=0$: otherwise P{\L} would keep $\norm{\nabla g_x}$ bounded away from zero and the integral of its speed would be infinite. Thus $y^\star\in Y^\star(x)$, and therefore
\[
 \dist(y,Y^\star(x))
 \le \sqrt{\frac{2\kappa g_x(y)}{\ell}}
 \le \frac\kappa \ell\norm{\nabla g_x(y)},
\]
where the second inequality is the P{\L} inequality. This proves \eqref{eq:inner-error-bound}.

Choose the maximizer supplied by \eqref{eq:inner-error-bound}. The gradient identity just proved and \cref{ass:smoothness} give
\[
 \norm{\nabla_xf(x,y)-\nabla\Phi(x)}
 \le \ell\norm{y-y^\star}
 \le \kappa\norm{\nabla_yf(x,y)},
\]
which is \eqref{eq:primal-gradient-bias}.

Finally, fix $x,x'$ and $y^\star\in Y^\star(x)$. Since $\nabla_yf(x,y^\star)=0$, \cref{ass:smoothness} gives $\norm{\nabla_yf(x',y^\star)}\le \ell\norm{x'-x}$. Applying \eqref{eq:inner-error-bound} to the inner problem at $x'$ yields a maximizer $\bar y\in Y^\star(x')$ such that $\norm{\bar y-y^\star}\le\kappa\norm{x'-x}$. Thus
\[
 \begin{aligned}
 \norm{\nabla\Phi(x')-\nabla\Phi(x)}
 &=\norm{\nabla_xf(x',\bar y)-\nabla_xf(x,y^\star)}\\
 &\le \ell\bigl(\norm{x'-x}+\norm{\bar y-y^\star}\bigr)
 \le \ell(1+\kappa)\norm{x'-x}.
 \end{aligned}
\]
This proves \eqref{eq:outer-smoothness}.
\end{proof}

%% file: appendices/high_ratio_proofs.tex
\section{Proofs for the high-ratio lower bound}
\label{app:high-ratio}

We prove the construction lemmas used in Section~\ref{sec:high-ratio}. The lower bound for $f_{\mathrm{unstable}}$ accounts for initially large dual steps. On the other hand, the corridor function $g_{\mathrm{corridor}}$ controls the remaining primal movement, and the noisy function $g_\sigma$ supplies the stochastic term in the lower bound. Finally, we combine the corridor function with the noisy function to obtain a single NC-P{\L} game instance and prove its lower bound.

\subsection{Proof of the lower bound for \texorpdfstring{$f_{\mathrm{unstable}}$}{f\_unstable}}
\label{app:quadratic}

\begingroup
\renewcommand{\theHequation}{restated.quadratic-gradient.\arabic{equation}}
\quadraticGradientLowerStatement*
\endgroup

\begin{proof}
We use the same initialization for both update rules. We show that two consecutive updates with dual step sizes at least $8/\ell$ take any state $(x,y)$ satisfying $0\le x\le y$ to a state $(x'',y'')$ satisfying $0\le x\le x''\le y''$. For $m\ge1$, the first update gives $x_1=x_0$ and $y_1\ge2x_0$. Repeating the two-update argument then gives $x_t\ge x_0$ for $t=1,3,5,\ldots$ with $t\le m$.

Set $\mu=\ell/\kappa$, $x_0=4\sqrt{\Delta/(\ell+\mu)}$, and $y_0=0$. The Hessian is $(\ell+\mu)\bigl(\begin{smallmatrix}0&1/4\\1/4&-1/2\end{smallmatrix}\bigr)$, and the inner maximum is attained at $y=x/2$. Hence
\begin{align*}
 \norm{\nabla^2 f_{\mathrm{unstable}}}_{\mathrm{op}}
 &=\frac{1+\sqrt2}{4}(\ell+\mu)
 \le\frac{3(1+\sqrt2)}8\ell<\ell,\\
 \Phi_{\mathrm{unstable}}(x)-f_{\mathrm{unstable}}(x,y)
 &=\frac{\ell+\mu}{4}(y-x/2)^2,\\
 \frac12|\nabla_y f_{\mathrm{unstable}}(x,y)|^2
 &=\frac{\ell+\mu}{2}\bigl(\Phi_{\mathrm{unstable}}(x)-f_{\mathrm{unstable}}(x,y)\bigr).
\end{align*}
Since $(\ell+\mu)/2\ge\mu$, the game satisfies inner $\mu$-P{\L}. Its outer objective is differentiable and lower bounded. Both initial gaps equal $\Delta$, so the initialization is admissible because $c_{\mathrm{init}}\ge1$.

Write $s_t=(\ell+\mu)\eta_t/2\ge4$ for $t<m$. A step of normalized size $a$ sends $(x,y)$ to
\begin{equation}
 x'=x-\frac{a}{2\rho}y,
 \qquad
 y'=\begin{cases}
 (1-a)y+\frac a2x,&\mathrm{Sim},\\
 (1-a)y+\frac a2x',&\mathrm{Alt}.
 \end{cases}
 \label{eq:quadratic-updates-app}
\end{equation}
If $m\ge1$, the first step gives $x_1=x_0$ and $y_1=s_0x_0/2\ge2x_0$. Consider two subsequent steps $a,b\ge4$ starting from $0\le x\le y$, with states $(x',y')$ and $(x'',y'')$. Since $x'\le x\le y$, both update rules satisfy
\begin{align*}
 y'&\le(1-a/2)y\le0,\\
 x''-x&=-\frac{ay+by'}{2\rho}
 \ge\frac{ab/2-a-b}{2\rho}y\ge0.
\end{align*}
The last coefficient is nonnegative because $(a-2)(b-2)\ge4$. For \SGDAAlt{}, $y''=(b/2)x''+(1-b)y'\ge2x''$, since $x''\ge0$ and $y'\le0$. For \SGDASim{},
\begin{align*}
 x'-y'
 &=(1-a/2)x+(a-1-a/(2\rho))y
 \ge\frac a2(1-\rho^{-1})x\ge0,\\
 y''-x''
 &=(b/2-1)x'+(1-b+b/(2\rho))y'
 \ge-\frac b2(1-\rho^{-1})y'\ge0.
\end{align*}
Here the first line uses $y\ge x$ and $a-1-a/(2\rho)\ge0$; the second uses $x'\ge y'$ and $b/2-1\ge0$. Thus either pair of steps gives $0\le x\le x''\le y''$.

Induction from $(x_1,y_1)$ yields $x_t\ge x_0$ at every odd $t\le m$. Including $t=0$, we obtain
\begin{equation}
 \frac1T\sum_{t=0}^{T-1}|\Phi_{\mathrm{unstable}}'(x_t)|^2
 \ge\frac{(\ell+\mu)^2}{64T}\sum_{t=0}^{m}x_t^2
 \ge\frac{(\ell+\mu)^2x_0^2}{64T}\left(1+\left\lceil\frac m2\right\rceil\right)
 \ge\frac{m+1}{8T}\ell\Delta.
 \label{eq:unstable-conversion-app}
\end{equation}
The same bound holds directly for $m=0$. Exact gradients are admissible for every $\sigma\ge0$, completing the proof.
\end{proof}

\Needspace{10\baselineskip}
\subsection{Construction of \texorpdfstring{$g_{\mathrm{corridor}}$}{g\_corridor}}
\label{app:corridor}

\begingroup
\renewcommand{\theHequation}{restated.corridor.\arabic{equation}}
\corridorFunctionLowerStatement*
\endgroup

\begin{proof}
We prescribe the derivative before choosing its magnitude. An initial linear segment makes the first $m$ updates reach a constant-derivative interval. Its length is the distance traveled by the remaining steps; a final linear segment joins it to a minimum. Integrating the derivative then calibrates the exact initial gap.

Fix $m\in\{0,\ldots,T-1\}$ and let $S=\sum_{t=m}^{T-2}\eta_t$. When $S=0$, the minimum in the statement is $\ell\Delta$. Define
\begin{equation}
 P_0:=1,
 \qquad
 P_t:=\prod_{i=0}^{t-1}\left(1+\frac{\ell\eta_i}{\rho}\right)
 \quad(1\le t\le m),
 \label{eq:one-dimensional-product-app}
\end{equation}
\begin{equation}
 G^2:=\frac{\Delta}{S/\rho+\ell^{-1}(1-1/(2P_m^2))},
 \qquad
 G_0:=\frac{G}{P_m},
 \qquad
 b:=\frac{G-G_0}{\ell}+\frac{GS}{\rho}.
 \label{eq:one-dimensional-scales-app}
\end{equation}
For $G>0$, write $[z]_{[0,G]}:=\min\{G,\max\{0,z\}\}$. Add a constant to $g_{\mathrm{corridor}}$ so that $\inf_u g_{\mathrm{corridor}}(u)=0$, and define
\begin{equation}
 g_{\mathrm{corridor}}'(u):=-[G_0+\ell u]_{[0,G]}+[\ell(u-b)]_{[0,G]}.
 \label{eq:one-dimensional-function-app}
\end{equation}
The two intervals on which the clipping terms vary are disjoint, so $g_{\mathrm{corridor}}'$ is continuous and $\ell$-Lipschitz.

Starting from $u_0=0$, induction gives
\begin{align}
 u_t&=\frac{P_t-1}{\ell}G_0,
 &g_{\mathrm{corridor}}'(u_t)&=-P_tG_0,
 &&0\le t\le m,
 \label{eq:one-dimensional-first-app}\\
 u_t&=\frac{G-G_0}{\ell}+G\sum_{j=m}^{t-1}\frac{\eta_j}{\rho},
 &g_{\mathrm{corridor}}'(u_t)&=-G,
 &&m\le t\le T-1.
 \label{eq:one-dimensional-later-app}
\end{align}
Direct integration gives
\begin{equation}
 g_{\mathrm{corridor}}(0)
 =\frac{G^2-G_0^2}{2\ell}+\frac{G^2S}{\rho}+\frac{G^2}{2\ell}
 =G^2\left(\frac S\rho+\frac1\ell-\frac{1}{2\ell P_m^2}\right)=\Delta,
 \label{eq:one-dimensional-gap-app}
\end{equation}
and
\begin{equation}
 G^2\ge\frac{\Delta}{S/\rho+1/\ell}
 \ge\frac12\min\left\{\ell\Delta,\frac{\rho\Delta}{S}\right\}.
 \label{eq:one-dimensional-size-app}
\end{equation}
Since $P_m\ge1$, the denominator in \eqref{eq:one-dimensional-scales-app} is at least $S/\rho+1/(2\ell)$, so also $G^2\le2\min\{\ell\Delta,\rho\Delta/S\}$. This proves the two-sided comparison in the statement.

The game is the single-player objective $f(u)=g_{\mathrm{corridor}}(u)$, equivalently an instance with no inner variable. Hence the inner-P{\L} condition is vacuous. Equations \eqref{eq:one-dimensional-later-app} and \eqref{eq:one-dimensional-size-app}, summed over the last $T-m$ iterates, prove \eqref{eq:one-dimensional-average-main}. Since $G_0,G>0$ and $P_t>0$, the derivative is negative at every iterate, so each positive update strictly increases the scalar iterate. The initial gap equals $\Delta$ by \eqref{eq:one-dimensional-gap-app}. Exact gradients are admissible for every $\sigma$.
\end{proof}

\subsection{Construction of \texorpdfstring{$g_\sigma$}{g\_sigma}}

We construct $g_\sigma$ by giving each inner-gradient error its own component. The shifted quadratic game $g_\delta$ in Lemma~\ref{lem:shift-response} shows why this is useful: a shift in its inner gradients moves the outer variable away from its minimizer, and that displacement persists over a range of cumulative steps. We then build a smooth inner-P{\L} component $g_{\sigma,j}$ in which either sign of the error at update $j$ activates the same shift. The displacement estimate lets us count how many subsequent iterates have a large outer gradient. Summing the resulting squared-gradient bounds over components with enough remaining cumulative steps gives the lower bound on the average squared outer-gradient norm in Lemma~\ref{lem:noise-components-lower}. The fixed oracle in \eqref{eq:noise-oracle-main} generates all of these errors.

\Needspace{12\baselineskip}
\subsubsection{Persistent error under a deterministic shift}
\label{app:perturbation-response}

\begingroup
\renewcommand{\theHequation}{restated.shift-response.\arabic{equation}}
\shiftResponseStatement*
\endgroup

\begin{proof}
Recall that $g_\delta(z;a,b)=\ell za-\frac{\ell}{2\kappa}[a^2+(a+b+\delta)^2]$ and $(z_0,a_0,b_0)=(0,0,0)$. We bound the inner coordinates in terms of the largest outer iterate, then use an exact identity to close the outer bound. Once the initial inner transient has decayed, the same identity gives a positive lower bound on the outer displacement. Normalize the coordinates and dual steps by setting
\[
 X_n=\frac{\kappa z_n}{\delta},\qquad Y_n=\frac{a_n}{\delta},\qquad
 Z_n=1+\frac{b_n}{\delta},\qquad h_n=\frac{\ell\eta_n}{\kappa},\qquad
 \alpha=\frac{\kappa^2}{\rho}.
\]
The SGDA updates for $g_\delta$ become
\begin{equation}
 \begin{aligned}
 X_{n+1}&=X_n-\alpha h_nY_n,\\
 Y_{n+1}&=Y_n+h_n(X'_n-2Y_n-Z_n),\\
 Z_{n+1}&=Z_n-h_n(Y_n+Z_n),
 \end{aligned}
 \qquad (X_0,Y_0,Z_0)=(0,0,1),
 \label{eq:response-both-updates-app}
\end{equation}
where $X'_n=X_n$ for \SGDASim{} and $X'_n=X_{n+1}$ for \SGDAAlt{}. Write $\tau_n=\sum_{i<n}h_i$ for the cumulative normalized step. We choose universal positive constants $\alpha_\star,\bar h,\tau_-,\tau_+$ below. For $\alpha\le\alpha_\star$ and $h_n\le\bar h$, we will bound $|X_n|$ by a universal multiple of $\alpha$ when $\tau_n\le\tau_+/\alpha$, and prove $X_n\ge\alpha/4$ when $\tau_-\le\tau_n\le\tau_+/\alpha$.

Subtracting the updates for $Y_n$ and $Z_n$ in \eqref{eq:response-both-updates-app} yields $Y_{n+1}-Z_{n+1}-(Y_n-Z_n)=h_n(X'_n-Y_n)$. Summing this equality over $0\le i<n$, using $Y_0-Z_0=-1$, and then summing the update for $X_n$ gives
\begin{equation}
 X_n=\alpha\left[1+Y_n-Z_n-\sum_{i<n}h_iX'_i\right].
 \label{eq:response-exact-identity-app}
\end{equation}
The constant $1$ records the initial offset $Y_0-Z_0=-1$. To obtain a positive outer displacement, we must make the inner-coordinate difference small while keeping the accumulated outer term small.

Let $q_n=(Y_n,Z_n)^\trans\in\R^2$ collect the normalized inner coordinates in \eqref{eq:response-both-updates-app}. Its update is $q_{n+1}=(I+h_nB)q_n+h_nX'_n(1,0)^\trans$, where $B=\bigl(\begin{smallmatrix}-2&-1\\-1&-1\end{smallmatrix}\bigr)$. The eigenvalues of $-B$ are $(3\pm\sqrt5)/2$. Write $\lambda_0=(3-\sqrt5)/2$ and choose $\bar h\le2/(3+\sqrt5)$. Then $I+h B$ is positive semidefinite and $\norm{I+hB}_{\mathrm{op}}=1-\lambda_0h$ for $0\le h\le\bar h$.

Fix an endpoint $N$, and let $M_N=\max_{0\le i\le N}|X_i|$ be the largest magnitude of the normalized outer coordinate through iterate $N$. Since $|X'_n|\le M_N$ for $n<N$, the formula for $q_{n+1}$ and the norm bound on $I+h_nB$ give
\begin{align}
 \norm{q_{n+1}}
 &\le(1-\lambda_0h_n)\norm{q_n}+h_nM_N,\\
 \norm{q_n}
 &\le\prod_{i<n}(1-\lambda_0h_i)
 +\frac{M_N}{\lambda_0}\left[1-\prod_{i<n}(1-\lambda_0h_i)\right]
 \le e^{-\lambda_0\tau_n}+\frac{M_N}{\lambda_0}.
 \label{eq:response-scalar-induction-app}
\end{align}
Iterating the first inequality from $q_0=(0,1)^\trans$ gives the product bound in the second line; $1-a\le e^{-a}$ gives its exponential bound.

For every $n\le N$, $|Y_n-Z_n|\le\sqrt2\norm{q_n}$ and $\sum_{i<n}h_i|X'_i|\le\tau_NM_N$. Substituting these bounds and \eqref{eq:response-scalar-induction-app} into \eqref{eq:response-exact-identity-app}, then taking the maximum over $0\le n\le N$, gives
\begin{equation}
 M_N\le\alpha(1+\sqrt2)+\alpha\left(\frac{\sqrt2}{\lambda_0}+\tau_N\right)M_N.
 \label{eq:response-absorption-app}
\end{equation}
Put $C_X=2(1+\sqrt2)$. If $\alpha_\star\sqrt2/\lambda_0\le1/4$ and $\tau_+\le1/4$, then $\tau_N\le\tau_+/\alpha$ makes the coefficient multiplying $M_N$ in \eqref{eq:response-absorption-app} at most $1/2$. Thus $M_N\le C_X\alpha$, and \eqref{eq:response-scalar-induction-app} gives $\norm{q_n}\le e^{-\lambda_0\tau_n}+C_X\alpha/\lambda_0$ for $0\le n\le N$.

For a lower bound on $X_N$, use $Y_N-Z_N\ge-\sqrt2\norm{q_N}$ in \eqref{eq:response-exact-identity-app}, together with \eqref{eq:response-scalar-induction-app} and $M_N\le C_X\alpha$, to obtain
\begin{equation}
 \frac{X_N}{\alpha}
 \ge1-\sqrt2e^{-\lambda_0\tau_N}
       -\frac{\sqrt2}{\lambda_0}M_N-\tau_NM_N
 \ge1-\sqrt2e^{-\lambda_0\tau_N}
       -\frac{\sqrt2C_X}{\lambda_0}\alpha-C_X\alpha\tau_N.
 \label{eq:response-lower-identity-app}
\end{equation}
Choose $\tau_-=\lambda_0^{-1}\log(4\sqrt2)$, take $\tau_+\le\min\{1/4,1/(4C_X)\}$, and reduce $\alpha_\star$ so that
\[
 \alpha_\star\le\min\left\{
 \frac{\lambda_0}{4\sqrt2C_X},\,1,\,
 \frac{\tau_+}{2(\tau_-+2\bar h)}\right\}.
\]
For $\tau_-\le\tau_N\le\tau_+/\alpha$, the three subtracted terms $\sqrt2e^{-\lambda_0\tau_N}$, $(\sqrt2C_X/\lambda_0)\alpha$, and $C_X\alpha\tau_N$ in \eqref{eq:response-lower-identity-app} are each at most $1/4$, so $X_N\ge\alpha/4$. The choices also satisfy the conditions used in \eqref{eq:response-absorption-app} and give $\tau_-<\tau_+/\alpha$. Since $N$ was arbitrary, $|X_n|\le C_X\alpha$ holds when $\tau_n\le\tau_+/\alpha$, and $X_n\ge\alpha/4$ when $\tau_-\le\tau_n\le\tau_+/\alpha$. The argument permits $h_n=0$ and both definitions of $X'_n$ in \eqref{eq:response-both-updates-app}.

To express the bounds in the original coordinates $(z_n,a_n,b_n)$, choose $c>0$ no larger than $\bar h$, $\tau_+$, and $1/4$, and take $C\ge\tau_-$ large enough to bound $|X_n|/\alpha$ and $|Y_n|+|Z_n-1|$. Set $C_0=\alpha_\star^{-1}$. These constants are fixed before choosing the cutoff in Appendix~\ref{app:perturbation-realization} and the component scales in Appendix~\ref{app:one-perturbation-conversion}.

If $\rho\ge C_0\kappa^2$ and $\ell\eta_n\le c\kappa$, then $\alpha\le\alpha_\star$ and $h_n\le\bar h$. Substituting $z_n=\delta X_n/\kappa$, $a_n=\delta Y_n$, and $b_n=\delta(Z_n-1)$ into the bounds on $M_N$ and $\norm{q_n}$ gives
\[
 |z_n|\le C\frac{\kappa\delta}{\rho},\qquad
 |a_n|+|b_n|\le C\delta
 \quad\text{when}\quad
 \sum_{i<n}\eta_i\le c\frac{\rho}{\ell\kappa}.
\]
If $C\kappa/\ell\le\sum_{i<n}\eta_i\le c\rho/(\ell\kappa)$, then $\tau_-\le\tau_n\le\tau_+/\alpha$, and \eqref{eq:response-lower-identity-app} gives $z_n\ge c\kappa\delta/\rho$. No monotonicity of the steps was used.
\end{proof}

\subsubsection{A smooth NC-P{\L} game with the same mechanism}
\label{app:perturbation-realization}

The shifted quadratic starts with a nonzero inner gap at the origin. We introduce a third inner coordinate $s$ so that a gradient error can activate the shift from an inner maximizer. We also cut off the shift when $|a+b|$ is large to bound the Hessian globally while preserving the quadratic updates along the trajectory segment in \cref{lem:shift-response}. Fix $R_0$ larger than the inner-coordinate bound in that lemma, and choose a smooth cutoff $\chi:\R\to[0,1]$ equal to one on $[-R_0,R_0]$, zero outside $[-2R_0,2R_0]$, with $|\chi'|\le1/4$ and bounded first two derivatives. The constant $R_0$ is fixed before all smoothness and P{\L} constants below.

For a provisional condition parameter $K\ge1$, define the residual $r_K:\R^3\to\R$ and game $g_K:\R\times\R^3\to\R$ by
\begin{equation}
 r_K(a,b,s)=a+b+\sqrt K\,\chi\!\left(\frac{a+b}{\sqrt K}\right)\sin^2\!\left(\frac{\pi s}{2}\right),
 \qquad g_K(u;a,b,s)=ua-\frac{a^2+r_K(a,b,s)^2}{2K}.
 \label{eq:raw-perturbation-game}
\end{equation}
The profile vanishes at $s=0$ and has zero derivative there. At $s=\pm1$, it supplies the same shift and again has zero derivative. Thus either sign of a unit perturbation activates the shift, and the activation coordinate then stays fixed. Without the cutoff, the second derivative of $g_K$ in $s$ would grow linearly with $a+b$.

\begin{lemma}[Smoothness and inner P{\L} of the single-perturbation game]
\label{lem:raw-perturbation-game}
There are universal $c,C>0$ such that, for every $K\ge1$, $g_K$ has a $C$-Lipschitz gradient,
\begin{equation}
 \max_{a,b,s}g_K(u;a,b,s)=\frac K2u^2,
 \qquad \frac12\norm{\nabla_{a,b,s}g_K}^2
 \ge\frac cK\left[\frac K2u^2-g_K(u;a,b,s)\right].
 \label{eq:raw-perturbation-properties}
\end{equation}
The origin is stationary.
\end{lemma}

\begin{proof}
Completing the square gives
\begin{equation}
 g_K(u;a,b,s)=\frac K2u^2-\frac{(a-Ku)^2+r_K(a,b,s)^2}{2K}.
 \label{eq:raw-perturbation-gap}
\end{equation}
Both residuals vanish at $(a,b,s)=(Ku,-Ku,0)$, proving the outer-objective formula and attainment. For the P{\L} inequality, write $r=r_K(a,b,s)$, $e=a-Ku$, and
\[
 p=\partial_a r_K=\partial_b r_K
 =1+\chi'\!\left(\frac{a+b}{\sqrt K}\right)\sin^2\!\left(\frac{\pi s}{2}\right)
 \in[3/4,5/4].
\]
The restriction on $\chi'$ keeps $p$ bounded away from zero, so the gradients in $a$ and $b$ control both residuals $e$ and $r$ without using the gradient in $s$. Indeed, $\partial_a g_K=-(e+pr)/K$ and $\partial_b g_K=-pr/K$. Since $e=(e+pr)-pr$ and $r^2\le16(pr)^2/9$,
\begin{align}
 e^2+r^2
 &\le2(e+pr)^2+\left(2+\frac{16}{9}\right)(pr)^2
 \le4\bigl[(e+pr)^2+(pr)^2\bigr],\\
 \frac12\norm{\nabla_{a,b,s}g_K}^2
 &\ge\frac{e^2+r^2}{8K^2}
 =\frac1{4K}\left[\frac K2u^2-g_K(u;a,b,s)\right].
 \label{eq:perturbation-map-lower}
\end{align}
This proves the inner-P{\L} inequality.

For smoothness, $|\partial_a r_K|,|\partial_b r_K|\le C$ and $|\partial_s r_K|\le C\sqrt K$. Wherever a second derivative is nonzero, the cutoff implies $|a+b|\le2R_0\sqrt K$ and $|r_K|\le C\sqrt K$. Direct differentiation gives second-derivative bounds of order $K^{-1/2}$ in $(a,b)$, order one for mixed derivatives with $s$, and order $K^{1/2}$ for $\partial_{ss}r_K$. Hence, for $y=(a,b,s)$,
\begin{equation}
 \norm{\nabla_y^2(r_K^2/(2K))}
 \le\frac{\norm{\nabla_y r_K}^2+|r_K|\norm{\nabla_y^2r_K}}{K}
 \le C.
 \label{eq:perturbation-map-smooth}
\end{equation}
Outside that region the residual is affine, so the same bound holds globally. The remaining terms $ua-a^2/(2K)$ also have bounded Hessian for $K\ge1$. All first derivatives vanish at the origin.
\end{proof}

At $s=\pm1$ and $|a+b|\le R_0\sqrt K$, the residual satisfies
\begin{equation}
 r_K(a,b,s)=a+b+\sqrt K,\qquad \partial_s r_K(a,b,s)=0.
 \label{eq:perturbation-active-map-app}
\end{equation}
Thus $s$ stays fixed and the other coordinates follow the shifted quadratic of \cref{lem:shift-response} while the cutoff equals one. At $u=a=b=0$, a unit perturbation of $s$ creates inner gap $1/2$, since $g_K(0;0,0,\pm1)=-1/2$. A unit perturbation of $b$, with $u=a=s=0$, instead creates inner gap $1/(2K)$. The third coordinate therefore increases the gap by a factor $K$ for the same perturbation magnitude. On the region where the cutoff equals one, the quadratic curvature in $(a,b)$ remains of order $1/K$.

\subsubsection{Changing variables and the contribution of a single perturbation}
\label{app:one-perturbation-conversion}

Write $c_r,C_r,C_{\rm resp}$ for $c,C,C_0$ in \cref{lem:shift-response}, and enlarge $C_{\rm resp}$ so that $c_rC_{\rm resp}\ge2(C_r+2c_r)$. Let $c_0,C_0$ be the constants in \cref{lem:raw-perturbation-game}. Choose universal $\gamma,\theta>0$, fixed after the low-ratio construction, such that
\begin{equation}
 \theta C_0\le1,
 \qquad
 \frac{\theta c_0}{\gamma}\ge1.
 \label{eq:high-ratio-scale-choice}
\end{equation}
For condition number $\kappa$, set
\begin{equation}
 K:=\gamma\kappa,\qquad L:=\theta \ell.
 \label{eq:high-ratio-rescaling-parameters}
\end{equation}
In \eqref{eq:aggregate-noise-game}, set $g_{\sigma,j}\equiv0$ whenever $\eta_j\sigma=0$; these components have zero gradients and gaps.
Increase the universal lower bound on $\kappa$ so that $K\ge1$. For every block $j\in\{0,\ldots,T-2\}$ with $\eta_j\sigma>0$, write $y[j]=(a[j],b[j],s[j])$ and define
\begin{equation}
 g_{\sigma,j}(v[j];a[j],b[j],s[j])
 :=L(\eta_{j}\sigma)^2
 g_K\left(\frac{v[j]}{\eta_{j}\sigma};
 \frac{a[j]}{\eta_{j}\sigma},\frac{b[j]}{\eta_{j}\sigma},\frac{s[j]}{\eta_{j}\sigma}\right).
 \label{eq:rescaled-perturbation-block-app}
\end{equation}
The squared amplitude compensates exactly for the spatial rescaling. Writing $w\in\R^4$ for the component coordinates, we have
\begin{equation}
 \nabla_w^2\!\left[(\eta_j\sigma)^2g_K\!\left(\frac{w}{\eta_j\sigma}\right)\right]
 =\nabla^2 g_K\!\left(\frac{w}{\eta_j\sigma}\right).
 \label{eq:noise-rescaling-hessian}
\end{equation}
Thus changing the positive amplitude $\eta_j\sigma$ does not change the Hessian bound.
Define the shift as a function of the perturbed inner coordinate
\begin{equation}
 \delta_j(s)=\sqrt K\,\eta_j\sigma\sin^2\!\left(\frac{\pi s}{2\eta_j\sigma}\right).
 \label{eq:noise-displacement-main}
\end{equation}
The profile and its derivative vanish at zero, while $\delta_j(\pm\eta_j\sigma)=\sqrt K\,\eta_j\sigma$ and $\delta_j'(\pm\eta_j\sigma)=0$. Thus either sign of the gradient error activates the same shift and leaves its derivative zero. Writing $\chi_j(q):=\chi(q/(\sqrt K\eta_j\sigma))$ gives exactly the component formula \eqref{eq:noise-component-main}. Its displacement $\chi_j(a[j]+b[j])\delta_j(s[j])$ depends on the inner coordinates, and the cutoff keeps second derivatives bounded globally. Once $s[j]$ is fixed at either activated value, the updates of $(v[j],a[j],b[j])$ agree with those of $g_\delta$ in \eqref{eq:local-perturbation-game-main}, with parameters $(L,K,\sqrt K\eta_j\sigma)$, wherever $|a[j]+b[j]|\le R_0\sqrt K\eta_j\sigma$.

Its full-gradient Lipschitz constant is at most $L C_0\le \ell$, and its inner-P{\L} parameter is at least $L c_0/K\ge \ell/\kappa$. Also $\Phi_{\sigma,j}(v[j])=L K(v[j])^2/2$, and the origin is both an outer minimizer and an inner maximizer. These properties hold for every positive amplitude, including the initial large steps; the step-size restriction in \cref{lem:one-perturbation-main} is needed only for its trajectory bound.

\Needspace{13\baselineskip}

\begingroup
\renewcommand{\theHequation}{restated.one-perturbation.\arabic{equation}}
\onePerturbationStatement*
\endgroup

\begin{proof}
Once the error activates a shift, the component agrees with the shifted quadratic as long as its inner coordinates stay inside the cutoff. We verify this agreement over the displacement window, then use the non-increasing steps to count the iterates in that window.

Let $\xi$ take values $\pm1$ with equal probability. All coordinates remain zero until update $j$, which sends $s[j]$ to $\eta_j\sigma\xi$ and leaves the other coordinates zero. At either sign, $\delta_j(s[j])=\sqrt K\eta_j\sigma$ and $\delta_j'(s[j])=0$. Thus $s[j]$ stays fixed and both signs activate the same shift.

Choose
\[
 c_{\rm high}=\gamma^2C_{\rm resp},\qquad
 c_{\rm step}=\frac{c_r\gamma}{\theta},
\]
and take the constant in the subsequent-sum condition at least $c_r/(\theta\gamma)$. Assume that this condition holds. For $j+1\le t\le T-1$, write
\[
 \tau_t=\frac{L}{K}\sum_{r=j+1}^{t-1}\eta_r.
\]
By the hypotheses, $\rho\ge C_{\rm resp}K^2$, each increment of $\tau_t$ is at most $L\eta_j/K\le c_r$, and $\tau_{T-1}\ge c_r\rho/K^2$.

Consider the exact-gradient trajectory $(v_t,a_t,b_t)$ of the uncut shifted quadratic in \eqref{eq:local-perturbation-game-main}, with parameters $(L,K,\sqrt K\eta_j\sigma)$, starting at the origin at time $j+1$. Lemma~\ref{lem:shift-response} gives
\begin{align*}
 v_t&\ge c_r\frac{K^{3/2}\eta_j\sigma}{\rho}
 &&\text{when }C_r\le\tau_t\le\frac{c_r\rho}{K^2},\\
 |a_t|+|b_t|&\le C_r\sqrt K\eta_j\sigma
 &&\text{when }\tau_t\le\frac{c_r\rho}{K^2}.
\end{align*}
Let $i$ and $k$ be the first indices with $\tau_i\ge C_r$ and $\tau_k\ge c_r\rho/K^2$. They exist because $\tau_{T-1}\ge c_r\rho/K^2$; moreover, $c_r\rho/K^2\ge2(C_r+2c_r)$ and each increment is at most $c_r$, so $i<k\le T-1$.

For $j+1\le t<k$, the quadratic trajectory satisfies $|a_t+b_t|\le C_r\sqrt K\eta_j\sigma<R_0\sqrt K\eta_j\sigma$. Starting from the common state at time $j+1$, induction therefore identifies it with the actual component trajectory through iterate $k-1$. Every gradient query needed for this induction lies where the cutoff is identically one. For \SGDAAlt{}, the inner-gradient query after the primal update has the same $a+b$, so the agreement also holds at that query. Thus $v_t[j]=v_t$ on the interval we use below.

Since $\eta_t\le\eta_j$ for $t>j$ and the first crossing of $C_r$ overshoots by at most $c_r$,
\begin{equation}
 (k-i)\frac{L\eta_j}{K}
 \ge\tau_k-\tau_i
 \ge\frac{c_r\rho}{K^2}-C_r-c_r
 \ge\frac{c_r\rho}{2K^2}.
 \label{eq:response-output-count-app}
\end{equation}
For every $i\le t<k$, the displacement bound and $\Phi_{\sigma,j}'(v)=LKv$ give $\Phi_{\sigma,j}'(v_t[j])\ge c_rLK^{5/2}\eta_j\sigma/\rho$. These are $k-i$ reported iterates, so
\begin{align}
 \sum_{t<T}\E|\Phi_{\sigma,j}'(v_t[j])|^2
 &\ge(k-i)\left(\frac{c_rLK^{5/2}\eta_j\sigma}{\rho}\right)^2
 \ge\frac{c_r^3}{2}\frac{LK^4\sigma^2\eta_j}{\rho}
 \ge c\frac{\ell\kappa^4\sigma^2\eta_j}{\rho}.
 \label{eq:single-perturbation-conversion-app}
\end{align}

All constants are independent of the horizon and step-size sequence, and the same displacement estimate holds for either noise sign. In the aggregate game, the fixed oracle constructed in Lemma~\ref{lem:noise-components-lower} generates each component's prescribed perturbation at its designated update.
\end{proof}

\Needspace{18\baselineskip}
We now realize the prescribed perturbations through one fixed oracle and sum their contributions. Only components with enough subsequent cumulative steps contribute to the estimate; the omitted components account for a bounded total step size, leaving a constant fraction of $S$ when $S$ is sufficiently large.

\begingroup
\renewcommand{\theHequation}{restated.noise-components.\arabic{equation}}
\begin{restatable}[Lower bound from the noise components]{lemma}{noiseComponentsStatement}
\label{lem:noise-components-lower}
There are universal $c,C>0$ with the following property. Fix $\ell,\sigma>0$, $\kappa\ge\kappa_{\rm high}$, $\rho\ge\max\{1,c_{\rm high}\kappa^2\}$, an integer $T\ge2$, a predetermined schedule $\eta_0\ge\cdots\ge\eta_{T-2}\ge0$, and $m\in\{0,\ldots,T-1\}$ with $\ell\eta_t<8$ for $m\le t\le T-2$. Set $S=\sum_{t=m}^{T-2}\eta_t$. The components in \eqref{eq:aggregate-noise-game} can be chosen $\ell$-smooth and NC-$(\ell/\kappa)$-P{\L}, with $\Phi_{\sigma,j}(0)=\min\Phi_{\sigma,j}=g_{\sigma,j}(0,0)=0$. Both \SGDASim{} and \SGDAAlt{} on $g_\sigma$, initialized at the origin with exact outer gradients and the fixed inner oracle $G_y(v,y;\xi)=\nabla_y g_\sigma(v,y)+\sigma\xi e(y)$ from Section~\ref{sec:high-ratio-noisy}, satisfy
\begin{equation}
 \frac1T\sum_{t=0}^{T-1}\E\norm{\nabla\Phi_\sigma(v_t)}^2
 \ge c\frac{\kappa^4\ell\sigma^2S}{\rho T}
 \quad\text{if }S\ge\frac{C\rho}{\kappa\ell}.
 \label{eq:noise-components-main}
\end{equation}
\end{restatable}
\endgroup

\begin{proof}
Let $J=\{0,\ldots,T-2\}$. Set $g_{\sigma,j}\equiv0$ for $\eta_j=0$, and choose every positive-step block from \eqref{eq:rescaled-perturbation-block-app}. The preceding smoothness, inner-P{\L}, and initialization estimates apply to all these blocks, with no restriction on the initial step sizes.

The selector $e(y)$ from Section~\ref{sec:high-ratio-noisy} chooses the $s[j]$-coordinate of the smallest $j\in J$ with $s[j]=0$, or is zero if no such index exists. Its selection regions are finite intersections of coordinate hyperplanes and their complements, so it is measurable and has norm at most one. Fresh independent uniform signs make the inner oracle pointwise unbiased, with variance $\sigma^2\norm{e(y)}^2\le\sigma^2$.

We next verify that this oracle activates the intended component at each positive update. The induction invariant is that, at the start of a positive update $j$, all blocks with index at least $j$ remain at the origin, while every earlier block has $s[i]=\pm\eta_i\sigma\ne0$. Stationarity at the origin keeps the unactivated blocks fixed during the primal update. Thus $e(y)$ selects the third inner coordinate of block $j$ for either update order, and update $j$ sends $s[j]$ to $\pm\eta_j\sigma$ while leaving its other coordinates at zero.

At these peaks, $\partial_s r_K(a,b,\pm1)=0$ for every $a,b$, including outside the cutoff region. All subsequent exact $s[j]$-gradients therefore vanish, so $s[j]$ stays fixed even if a large step changes the cutoff factor. The remaining unactivated blocks stay at the origin, completing the induction. Since the schedule is non-increasing, zero steps are trailing. The selector then repeatedly chooses the first zero-step block, but the zero step sizes leave every iterate unchanged. Consequently, each nonzero block receives one equiprobable $\pm\sigma$ error at its assigned update and exact gradients otherwise, with independent signs across blocks.

Since the coordinates are disjoint,
\[
 \norm{\nabla\Phi_\sigma(v_t)}^2
 =\sum_{j\in J}\norm{\nabla\Phi_{\sigma,j}(v_t[j])}^2.
\]
Let $C_{\rm one}$ be the remaining-step constant in \cref{lem:one-perturbation-main}, and enlarge $\kappa_{\rm high}$ so that $8\le c_{\rm step}\kappa$. Every positive step with index $t\ge m$ then satisfies its step-size condition. Define the tail indices $J_{\rm tail}:=\{m,\ldots,T-2\}$ and $J_{\rm good}:=\{j\in J_{\rm tail}:\sum_{t=j+1}^{T-2}\eta_t\ge C_{\rm one}\rho/(\kappa\ell)\}$. Non-increasingness ensures $\eta_j>0$ for every $j\in J_{\rm good}$. If $J_{\rm tail}\setminus J_{\rm good}$ is nonempty, its first index $j_0$ satisfies
\[
 \sum_{j\in J_{\rm tail}\setminus J_{\rm good}}\eta_{j}
 \le\sum_{t=j_0}^{T-2}\eta_t
 <\frac8\ell+\frac{C_{\rm one}\rho}{\kappa\ell}
 \le\frac{C_1\rho}{\kappa\ell},
 \qquad C_1:=C_{\rm one}+\frac8{c_{\rm high}},
\]
where the last inequality uses $\rho\ge c_{\rm high}\kappa^2$ and $\kappa\ge1$. If the excluded set is empty, its total step size is zero. Taking $C:=2C_1$, the condition $S\ge C\rho/(\kappa\ell)$ implies $\sum_{j\in J_{\rm good}}\eta_{j}\ge S/2$. Summing \eqref{eq:one-perturbation-gradient-bound} over these indices gives
\begin{align*}
 \frac1T\sum_{t=0}^{T-1}\E\norm{\nabla\Phi_\sigma(v_t)}^2
 &\ge c\frac{\kappa^4\ell\sigma^2}{\rho T}\sum_{j\in J_{\rm good}}\eta_{j}
 \ge\frac c2\frac{\kappa^4\ell\sigma^2S}{\rho T}.
\end{align*}
Reducing the universal constant $c$ proves the claim.
\end{proof}

\subsection{Proof of the lower bound for \texorpdfstring{$f_{\mathrm{stable}}$}{f\_stable}}
\label{app:tail-lower}

The aggregate construction already supplies a fixed oracle using only the inner coordinates. We add the deterministic corridor and combine its lower bound with the contributions of the noise components.

\Needspace{10\baselineskip}
\begingroup
\renewcommand{\theHequation}{restated.tail-lower.\arabic{equation}}
\tailLowerStatement*
\endgroup

\begin{proof}
For the chosen algorithm, \SGDASim{} or \SGDAAlt{}, use the predetermined non-increasing step sizes in the statement. Let $J=\{0,\ldots,T-2\}$, and choose the corridor from \cref{lem:one-dimensional-lower}. For $\sigma>0$, choose the aggregate noise game $g_\sigma$ and its fixed inner oracle from \cref{lem:noise-components-lower}; for $\sigma=0$, set $g_{\sigma,j}\equiv0$ for every $j\in J$ and use exact gradients. Use these games in \eqref{eq:high-separable-game-main}, take exact outer gradients, and initialize every coordinate at zero. In either case, the dimensions are $d_x=1+|J|=T$ and $d_y=3|J|=3(T-1)$.

We first verify the game assumptions. Each component has full-gradient Lipschitz constant at most $\ell$. Since the components use disjoint coordinates, summing the squared component inequalities gives
\[
 \norm{\nabla f_{\mathrm{stable}}(u,v;y)-\nabla f_{\mathrm{stable}}(\widetilde u,\widetilde v;\widetilde y)}^2
 \le\ell^2\bigl(|u-\widetilde u|^2+\norm{v-\widetilde v}^2+\norm{y-\widetilde y}^2\bigr).
\]
Thus the smoothness assumption holds without a dimension factor. Each inner maximum is finite and attained, so their finite product attains the maximum of $f_{\mathrm{stable}}$. Its outer objective is
\[
 \Phi_{\mathrm{stable}}(u,v)
 =g_{\mathrm{corridor}}(u)+\Phi_\sigma(v)
 =g_{\mathrm{corridor}}(u)+\frac{L K}{2}\sum_{\substack{j\in J\\\eta_j\sigma>0}}(v[j])^2.
\]
It is differentiable and lower bounded. Summing the component inner-P{\L} inequalities gives
\begin{align*}
 \frac12\norm{\nabla_y f_{\mathrm{stable}}(u,v;y)}^2
 &=\frac12\sum_{j\in J}\norm{\nabla_{y[j]}g_{\sigma,j}(v[j];y[j])}^2\\
 &\ge\frac\ell\kappa\sum_{j\in J}\bigl(\Phi_{\sigma,j}(v[j])-g_{\sigma,j}(v[j];y[j])\bigr)\\
 &=\frac\ell\kappa\bigl(\Phi_{\mathrm{stable}}(u,v)-f_{\mathrm{stable}}(u,v;y)\bigr).
\end{align*}
The corridor supplies outer gap $\Delta$, each perturbation component has zero outer gap at the origin, and every component is initially maximized in its inner variables. Hence the initial outer and inner gaps are $\Delta$ and zero. These conclusions also hold for the corridor alone, with a vacuous inner-P{\L} inequality.

Adding the corridor does not change any noise-component trajectory, because its coordinates are disjoint and the inner oracle depends only on $(v,y)$. The oracle verification in the proof of \cref{lem:noise-components-lower} therefore also applies to $f_{\mathrm{stable}}$. Together with the game assumptions above, this proves that the origin-initialized instance belongs to $\Class(\ell,\kappa,\Delta,\sigma)$.

It remains to prove the lower bound. Let $C$ be the threshold constant in \cref{lem:noise-components-lower}, and increase $\kappa_{\rm high}$ so that $\kappa_{\rm high}\ge C$. If $S<C\rho/(\kappa\ell)$, including $S=0$, the corridor gives $c\ell\Delta$, since $(T-m)/T>1/2$ and $S/\rho<1/\ell$. This proves the capped bound.

Otherwise $S>0$, and \cref{lem:noise-components-lower} applies when $\sigma>0$; when $\sigma=0$, the noise contribution is zero. Since the squared outer gradients of disjoint components add, their bounds give
\begin{equation}
 \frac1T\sum_{t=0}^{T-1}\E\norm{\nabla\Phi_{\mathrm{stable}}(u_t,v_t)}^2
 \ge c\left[
 \min\!\left\{\ell\Delta,\frac{\rho\Delta}{S}\right\}
 +\frac{\kappa^4\ell\sigma^2S}{\rho T}\right].
 \label{eq:combined-bounds-main}
\end{equation}
The bracket dominates the capped maximum in \eqref{eq:tail-lower-main}, proving the lemma.
\end{proof}

%% file: appendices/low_ratio_proofs.tex
\Needspace{8\baselineskip}
\section{Proofs for the low-ratio lower bound}
\label{app:low-ratio}

We prove the two properties used by \cref{thm:low-ratio}: small steps preserve a region of nonzero outer gradient, and the initially larger steps can be made to end inside that region. The first requires a smooth game with a stable periodic solution; the second modifies that specific game away from the solution. The final rescaling and initialization checks appear in Appendix~\ref{app:low-ratio-theorem}.

\subsection{Construction of \texorpdfstring{$g_{\mathrm{cycle}}$}{g\_cycle}}
\label{app:low-ratio-periodic}
\label{app:low-ratio-small-steps}

\begingroup
\renewcommand{\theHequation}{restated.low-periodic.\arabic{equation}}
\lowPeriodicGameStatement*
\endgroup

\begin{proof}
We first verify global smoothness and inner P{\L}, then check that the prescribed curve solves the GDA equations. A rotating frame separates phase shifts along the orbit from departures that decay toward it. We use this decay to construct a Lyapunov function whose sublevel set is invariant under both discrete updates.

\paragraph{Global smoothness and inner P{\L}.}
We construct $g_{\mathrm{cycle}}(u,a,b)=\sqrt2\,u^\trans a-\norm a^2-\norm{R(a,b)}^2/2$ on $\R^2\times\R^2\times\R^2$. Its affine dependence on $u$ will be essential to the ratio scaling. The residual has the form $R(a,b)=b+\chi_b(b)\psi(a)$: the values and first derivatives of $\psi$ produce the periodic solution, its second derivatives control departures from that solution, and the cutoffs enforce global smoothness.

Let $J=\bigl(\begin{smallmatrix}0&-1\\1&0\end{smallmatrix}\bigr)$ be rotation by a right angle. The target path is $u_\star(t)=(\cos t,\sin t)^\trans$, with $a_\star=-Ju_\star/\sqrt2$ and $b_\star=u_\star/\sqrt2$. Since $\dot u=-\sqrt2a$ and, where the cutoffs equal one, $\dot b=-(b+\psi(a))$, the first and third equations require $\psi(a_\star)=a_\star-Ja_\star$. The remaining equation requires $D\psi(a_\star)^\trans a_\star=-2a_\star+Ja_\star$. We satisfy these value and derivative conditions first, then check contraction through the remaining derivatives.

Choose a smooth cutoff $\chi_a:[0,\infty)\to[0,1]$ that equals one on a neighborhood of $\tfrac12$ and vanishes outside a slightly larger compact interval. Define the polynomials
\begin{equation}
 c(q):=1-3(q-\tfrac12)+\frac72(q-\tfrac12)^2,
 \qquad
 d(q):=-1+(q-\tfrac12),
 \label{eq:low-coefficients-app}
\end{equation}
and set
\begin{equation}
 \psi(a):=\chi_a(\norm a^2)\bigl[c(\norm a^2)a+d(\norm a^2)Ja\bigr].
 \label{eq:low-h-app}
\end{equation}
The orbit conditions require $c(1/2)=1$, $d(1/2)=-1$, and $c'(1/2)=-3$. The remaining choices $c''(1/2)=7$ and $d'(1/2)=1$ make the transverse linearization stable, as verified below. The map $\psi$ is smooth and compactly supported. Choose a smooth radial cutoff $\chi_b:\R^2\to[0,1]$ that equals one on a neighborhood of the circle $\norm b=1/\sqrt2$, vanishes outside a sufficiently large ball, and satisfies
\begin{equation}
 \norm{\psi}_\infty\norm{\nabla\chi_b}_\infty\le\frac12.
 \label{eq:low-b-cutoff-app}
\end{equation}
\begin{samepage}
Define
\begin{equation}
 R(a,b):=b+\chi_b(b)\psi(a),
 \label{eq:low-map-app}
\end{equation}
and
\begin{equation}
 g_{\mathrm{cycle}}(u,a,b)
 :=\sqrt2\,u^\trans a-\norm a^2-\frac12\norm{R(a,b)}^2.
 \label{eq:low-cycle-game-app}
\end{equation}
\end{samepage}

For every fixed $a$ and target $w\in\R^2$, the map $b\mapsto w-\chi_b(b)\psi(a)$ is a contraction with factor at most $1/2$ by \eqref{eq:low-b-cutoff-app}. Its unique fixed point solves $R(a,b)=w$. Thus $b\mapsto R(a,b)$ is bijective. Moreover,
\[
 D_bR=I+\psi(a)\nabla\chi_b(b)^\trans,
 \qquad \norm{D_bR-I}\le\tfrac12,
 \qquad \norm{(D_bR)^{-1}}\le2.
\]
The inverse function theorem gives a smooth inverse. Thus $T(a,b):=(\sqrt2a,R(a,b))$ is a smooth bijection. Its triangular derivative has diagonal blocks $\sqrt2I_2$ and $D_bR$; boundedness of $D_aR$ and $(D_bR)^{-1}$ gives $\norm{(DT)^{-1}}\le C$, and hence $\norm{DT^\trans q}\ge C^{-1}\norm q$ for every $(a,b)$ and $q\in\R^4$.

All derivatives of $\psi$ are bounded. Whenever a second derivative of $R$ is nonzero, both $a$ and $b$ lie in fixed compact sets; elsewhere $R$ is affine. Hence $DR$ and $R D^2R$ are globally bounded, and so is $\nabla^2\norm{R}^2/2$. The full gradient of $g_{\mathrm{cycle}}$ is therefore universally Lipschitz.

For each $a$, maximizing over $b$ cancels the residual. Maximizing over $a$ then gives $\Phi_{g_{\mathrm{cycle}}}(u)=\norm u^2/2$ and
\begin{equation}
 \Phi_{g_{\mathrm{cycle}}}(u)-g_{\mathrm{cycle}}(u,a,b)=\frac12\norm{u-\sqrt2a}^2+\frac12\norm{R(a,b)}^2.
 \label{eq:low-gap-main}
\end{equation}
The surjectivity and transposed-derivative bound for $T$ verify the hypotheses of \cref{lem:completion-square} with outer vector $(u,0)$, giving inner $\mu$-P{\L} for a universal $\mu>0$.

\paragraph{The periodic orbit and its attraction.}
Along \eqref{eq:low-cycle-main}, $b_\star=Ja_\star$, $\norm{a_\star}^2=\tfrac12$, and both cutoffs are constant. Hence
\[
 \psi(a_\star)=a_\star-Ja_\star,
 \qquad
 R(a_\star,b_\star)=a_\star.
\]
On the neighborhood where both cutoff functions equal one, writing $q=\norm a^2$,
\begin{equation}
 D\psi(a)^\trans a=[c(q)+2q c'(q)]a-d(q)Ja.
 \label{eq:low-derivative-identity-app}
\end{equation}
The values $c(\tfrac12)=1$, $c'(\tfrac12)=-3$, and $d(\tfrac12)=-1$ give
\[
 \dot u_\star=-\sqrt2a_\star=Ju_\star,
 \quad
 \dot a_\star=\sqrt2u_\star-2a_\star-D\psi(a_\star)^\trans a_\star=Ja_\star,
 \quad
 \dot b_\star=-R(a_\star,b_\star)=Jb_\star.
\]
Thus \eqref{eq:low-cycle-main} solves the GDA differential equation and is $2\pi$-periodic.

Along this solution, $-\nabla_u g_{\mathrm{cycle}}=Ju_\star$, $\nabla\Phi_{g_{\mathrm{cycle}}}=u_\star$, and
$\norm{\nabla_{a,b}g_{\mathrm{cycle}}}^2=\norm{\dot a_\star}^2+\norm{\dot b_\star}^2=1$.
By \eqref{eq:low-gap-main}, the inner gap equals $5/4$. Inner ascent decreases this gap at rate one, while the motion of $u$ increases it at the same rate: $\norm{\nabla_u g_{\mathrm{cycle}}}^2=1$ and $u_\star^\trans\dot u_\star=0$. These contributions cancel, so the gap remains positive despite global inner P{\L}.

The definitions imply $\psi(e^{tJ}a)=e^{tJ}\psi(a)$ and $R(e^{tJ}a,e^{tJ}b)=e^{tJ}R(a,b)$. Thus the derivative calculation at any point of the periodic solution is the rotation of the calculation at $t=0$. Direct differentiation at $t=0$, using $c''(\tfrac12)=7$ and $d'(\tfrac12)=1$, gives
\begin{equation}
 A_\star:=D\psi(a_\star(0))=\begin{pmatrix}1&0\\-1&-2\end{pmatrix},
 \qquad
 S_\star:=D_a[D\psi(a)^\trans R(a,b)]_{(a_\star(0),b_\star(0))}
 =\begin{pmatrix}-1&3\\3&2\end{pmatrix}.
 \label{eq:low-local-matrices-app}
\end{equation}
Write $z=(u,a,b)$, $z_\star(t)=(u_\star(t),a_\star(t),b_\star(t))$, and let $F=(-\nabla_u g_{\mathrm{cycle}},\nabla_{a,b}g_{\mathrm{cycle}})$ be the GDA vector field. Set $\mathcal J=\operatorname{diag}(J,J,J)$ and $\mathcal Q_t=\operatorname{diag}(e^{tJ},e^{tJ},e^{tJ})$. Rotational equivariance gives $F(\mathcal Q_tz)=\mathcal Q_tF(z)$, so $\widetilde z(t)=\mathcal Q_t^\trans z(t)$ satisfies $\dot{\widetilde z}=F(\widetilde z)-\mathcal J\widetilde z$. The periodic solution becomes the fixed point $z_\star(0)$, whose Jacobian is
\begin{equation}
 M_\star=DF(z_\star(0))-\mathcal J=
 \begin{pmatrix}
 -J&-\sqrt2I_2&0\\
 \sqrt2I_2&-2I_2-S_\star-J&-A_\star^\trans\\
 0&-A_\star&-I_2-J
 \end{pmatrix}.
 \label{eq:low-linearization-app}
\end{equation}
A direct determinant calculation gives
\begin{equation}
 \det(\lambda I-M_\star)
 =\lambda(\lambda^5+7\lambda^4+7\lambda^3+21\lambda^2+10\lambda+4).
 \label{eq:low-characteristic-app}
\end{equation}
For the quintic factor, the first column of the Routh array is $1,7,4,9/2,370/63,4$. It has no sign change, so the Routh--Hurwitz criterion excludes roots in the open right half-plane. The array has neither a zero first-column entry nor an all-zero row, excluding roots on the imaginary axis as well. Thus all five roots have strictly negative real parts. The remaining zero eigenvalue is simple and represents phase shifts $z_\star(t)\mapsto z_\star(t+s)$; its eigenvector is $\dot z_\star(0)=\mathcal Jz_\star(0)$.

To control departures from the orbit, choose a basis $S\in\R^{6\times5}$ of the stable invariant subspace of $M_\star$, and let $A_-$ be its restriction in this basis.

\Needspace{7\baselineskip}
For sufficiently small $\varepsilon>0$, the map $\mathcal X(\vartheta,\xi)=\mathcal Q_\vartheta(z_\star(0)+S\xi)$, with $\vartheta\in\R/(2\pi\mathbb Z)$ and $\xi\in\R^5$, $\norm\xi<\varepsilon$, gives smooth coordinates near $\Gamma$: the columns of $S$ complement its tangent direction, and the orbit is compact and embedded. Using $M_\star S=SA_-$ and Taylor expansion of $F$ at $z_\star(0)$, differentiation of these coordinates along GDA gives
\[
 (\dot\vartheta-1)\mathcal J(z_\star(0)+S\xi)
 +S(\dot\xi-A_-\xi)=O(\norm\xi^2).
\]
The coefficient matrix $[\mathcal J(z_\star(0)+S\xi),S]$ remains uniformly invertible for small $\xi$. Hence $\dot\vartheta=1+O(\norm\xi^2)$ and $\dot\xi=A_-\xi+O(\norm\xi^2)$, uniformly in the phase $\vartheta$. Since $A_-$ is Hurwitz, choose $P\succ0$ satisfying $A_-^\trans P+PA_-=-I$, and define $V(\mathcal X(\vartheta,\xi))=\xi^\trans P\xi$. On a smaller neighborhood,
\[
 V(z)\asymp\operatorname{dist}(z,\Gamma)^2,
 \qquad
 \dot V=-\norm\xi^2+O(\norm\xi^3)\le-cV
\]
for a universal $c>0$. This proves local exponential orbital stability.

\paragraph{An invariant neighborhood for discrete GDA.}
We use $\dot V\le-cV$ to obtain $V(z^+)\le(1-ch)V(z)+Ch^2$ for both update rules, then choose $h_\star$ so that a fixed sublevel set of $V$ is invariant. Choose a fixed small $\delta>0$ so that $\mathcal U:=\{V\le\delta\}$ is compactly contained in this neighborhood and in $\{\norm u\ge1/2\}$. Simultaneous GDA gives $z^+_{\mathrm{Sim}}=z+hF(z)$. In the alternating update, the change in $u$ is $-h\sqrt2a$, so its inner-gradient query changes only the $a$-gradient, by $-2ha$. Thus $z^+_{\mathrm{Alt}}-z^+_{\mathrm{Sim}}=(0,-2h^2a,0)$. For $z\in\mathcal U$ and sufficiently small $h$, both updates and the line segments joining them to $z$ remain in a fixed compact neighborhood where $V$ is smooth. Taylor expansion therefore gives, uniformly for either update,
\[
 V(z^+)\le V(z)+h\nabla V(z)^\trans F(z)+Ch^2
 \le(1-ch)V(z)+Ch^2
 \le(1-ch)\delta+Ch^2\le\delta,
\]
where we choose $h_\star>0$ small enough that $ch_\star\le1$ and $Ch_\star\le c\delta$. Thus $\mathcal U$ is invariant for every $h\in[0,h_\star]$, including $h=0$. Induction and $\nabla\Phi_{g_{\mathrm{cycle}}}(u)=u$ give the discrete conclusion for every such step sequence. All constants depend only on the fixed game.

For the next lemma, fix a bounded open neighborhood $\mathcal W\subset\R^4$ of the inner projection of $\mathcal U$. We will preserve the game for every $u$ whenever $(a,b)\in\mathcal W$. This also protects the alternating inner-gradient query after the primal update, since that update changes only $u$.
\end{proof}

\subsection{Construction of \texorpdfstring{$\tilde g_{\mathrm{cycle}}$}{tilde g\_cycle}}
\label{app:low-ratio-interpolation}
\label{app:low-ratio-arbitrary-steps}

To construct $\tilde g_{\mathrm{cycle}}$, we extend the base game to accommodate finitely many initial steps $h_t>h_\star$. At the initial query states $y_t$ to be chosen, for $0\le t<m$ and $e_1=(1,0)^\trans$, the required gradients are
\begin{equation}
 \nabla_u\tilde g_{\mathrm{cycle}}(e_1,y_t)=0,
 \qquad \nabla_y\tilde g_{\mathrm{cycle}}(e_1,y_t)=\frac{y_{t+1}-y_t}{h_t}.
 \label{eq:low-target-gradients-app}
\end{equation}
They keep the outer variable fixed and send each inner state to the next, ending on the cycle. A local change of inner coordinates prescribes these gradients through its derivative. We preserve $a$ to retain both the primal gradient $\sqrt2a$ and the bilinear term $\sqrt2u^\trans a$. A nonlinear change of these coordinates would introduce inner-Hessian terms proportional to $u$, obstructing global joint smoothness. The following lemma realizes the prescribed derivative inside a ball; disjoint balls will keep its bounds independent of the number of steps.

\begin{lemma}[Local derivative modification that leaves the first two coordinates unchanged]
\label{lem:low-local-modification}
Fix $M<\infty$ and $\delta>0$. There is $C=C(M,\delta)$ such that, for every $z\in\R^5$ and every
\begin{equation}
 N=\begin{pmatrix}I_2&0\\P&Q\end{pmatrix},
 \qquad
 \det N>0,
 \qquad
 \norm N+\norm{N^{-1}}\le M,
 \label{eq:low-target-derivative-app}
\end{equation}
there is a smooth bijection $G:\R^5\to\R^5$ with a smooth inverse. For every $w$, the first two coordinates of $G(w)$ equal those of $w$. Moreover, $G(z)=z$, $DG(z)=N$, $G$ equals the identity outside $B(z,\delta)$, $G$ maps this ball onto itself, and
\begin{equation}
 \norm{DG}_\infty+\norm{DG^{-1}}_\infty+\norm{D^2G}_\infty\le C.
 \label{eq:low-local-bounds-app}
\end{equation}
\end{lemma}

\begin{proof}
We join the identity matrix to $N$ through uniformly invertible matrices, then realize this matrix path as the derivative of a compactly supported flow. Since $\det N>0$ and the first diagonal block is $I_2$, the matrix $Q$ is invertible and has positive determinant. Write $Q=US$, where $U^\trans U=I$, $\det U=1$, and $S$ is symmetric positive definite. Choose a skew-symmetric matrix $A$ with $e^A=U$ and $\norm A\le\pi$, and define
\[
 Q_\tau=e^{\tau A}((1-\tau)I+\tau S),
 \qquad
 N_\tau=\begin{pmatrix}I_2&0\\\tau P&Q_\tau\end{pmatrix}.
\]
The matrices $N_\tau$, their inverses, and their derivatives are uniformly bounded in terms of $M$. Put $A_\tau=\dot N_\tau N_\tau^{-1}$. Choose a smooth radial cutoff $\chi$ that equals one on $B(z,\delta/3)$ and vanishes outside $B(z,2\delta/3)$, and set
\[
 X_\tau(w)=\chi(w)A_\tau(w-z).
\]
This time-dependent differential equation has zero first two components and uniformly bounded first two derivatives. Let $G_\tau(w)$ denote its solution at time $\tau$ starting from $w$ at time zero. Then $G_\tau$ fixes $z$, preserves the first two coordinates, and satisfies
\[
 \frac{d}{d\tau}DG_\tau(z)=A_\tau DG_\tau(z),
 \qquad DG_0(z)=I,
\]
so $DG_\tau(z)=N_\tau$. Differentiating this differential equation once and twice with respect to the initial point, and then applying Gronwall's inequality, gives \eqref{eq:low-local-bounds-app} for $G=G_1$ and its inverse. Since the vector field vanishes in a neighborhood of the boundary, the ball is mapped onto itself.
\end{proof}

\Needspace{22\baselineskip}

\begingroup
\renewcommand{\theHequation}{restated.low-initial-updates.\arabic{equation}}
\lowInitialUpdatesStatement*
\endgroup

\begin{proof}
We choose separated inner states whose required update vectors are uniformly bounded. At each state, a local coordinate modification realizes the desired gradient while preserving the first two inner coordinates. The residual representation then verifies inner P{\L}, and the disjoint modification regions give uniform smoothness.

The case $m=0$ uses $\tilde{g}_{\mathrm{cycle}}(u,v,r)=g_{\mathrm{cycle}}(u,v)$ and $y_0=(y_\star(0),0)$. All initialization quantities are bounded on the fixed periodic solution, and the initial primal gradient is $\sqrt2a_\star(0)=-Je_1$, of norm one. Suppose $m\ge1$ and write $y=(a,b,r)\in\R^2\times\R^2\times\R$. The extra coordinate $r$ separates the query states even when their first four coordinates agree. Its value is not itself penalized, so this separation need not increase the initial gap. Define
\begin{equation}
 \begin{aligned}
 r_t&=-\sum_{j=t}^{m-1}h_j\quad(0\le t<m),
 &r_m&=0,\\
 y_t&=(0,-Be_1,r_t)\quad(t<m),
 &y_m&=(a_\star(0),b_\star(0),0).
 \end{aligned}
 \label{eq:low-interpolation-points-app}
\end{equation}
Here $B$ is a fixed universal constant large enough that a neighborhood of $(a,b)=(0,-Be_1)$ is outside the support of the nonlinear part of $R$ and separated from $\mathcal W$. We protect the lift $\mathcal W\times\R$ in the inner space $(a,b,r)$. Let $\bar H(a,b,r)=((a,R(a,b)),r)$. At all $y_t$ with $t<m$, it is the identity in a fixed neighborhood.

Put $v_t=(y_{t+1}-y_t)/h_t=(v_t^a,v_t^b,1)$ for $0\le t<m$. Its first four coordinates are zero except at $t=m-1$; their magnitude at the last step is bounded by a universal constant since $h_{m-1}>h_\star$ and the first four endpoint coordinates are fixed. Let $\pi_b:\R^5\to\R^2$ project onto the two coordinates corresponding to $b$. We seek a smooth bijection $\widetilde H:\R^5\to\R^5$ that preserves $a$, fixes the initial query states $y_0,\ldots,y_{m-1}$, and agrees with $\bar H$ on $\mathcal W\times\R$. Define the game in terms of this map by
\begin{equation}
 \tilde{g}_{\mathrm{cycle}}(u,y)=\ip{(u,0)}{\widetilde T(y)}-\frac12\norm{\widetilde T(y)}^2,
 \qquad \widetilde T(a,b,r)=(\sqrt2a,\pi_b\widetilde H(a,b,r))\in\R^4.
 \label{eq:low-interpolated-game-app}
\end{equation}
At each initial query state $y_t$, $t<m$, we have $\widetilde T(y_t)=(0,-Be_1)$, so $\nabla_u\tilde g_{\mathrm{cycle}}(e_1,y_t)=0$. The required inner gradient is therefore equivalent to
\begin{equation}
 D\widetilde H(y_t)^\trans(\sqrt2e_1,Be_1,0)=v_t.
 \label{eq:low-derivative-gradient-app}
\end{equation}
This identity determines the row corresponding to $b_1$. We choose the remaining two rows to make the derivative uniformly invertible with positive determinant:
\begin{equation}
 D\widetilde H(y_t)=
 \begin{pmatrix}
 I_2&0&0\\
 (v_t^a-\sqrt2e_1)^\trans/B&(v_t^b)^\trans/B&1/B\\
 0&e_2^\trans&0\\
 0&-Be_1^\trans&0
 \end{pmatrix}.
 \label{eq:low-interpolation-derivative-app}
\end{equation}
The lower-right $3\times3$ block has determinant one, so the full derivative also has determinant one. It and its inverse are uniformly bounded, since $B$ is fixed and $v_t$ is uniformly bounded. Distinct $y_t$ with $t<m$ are separated by at least $h_\star$, because their scalar coordinates differ by sums of steps exceeding $h_\star$. Choose a universal ball radius smaller than one quarter of this separation and small enough that $\bar H$ is the identity on every ball and all balls avoid $\mathcal W\times\R$.

Apply \cref{lem:low-local-modification} in these disjoint balls. Each local map fixes its center, has the prescribed derivative, preserves the first two coordinates, and maps its ball onto itself. Let $G$ be their composition and set $\widetilde H=G\circ\bar H$. Since $\bar H$ is the identity near every center, \eqref{eq:low-interpolation-derivative-app} holds. At most one factor of $G$ is nonidentity along the image of any point, so the derivative, inverse-derivative, and second-derivative bounds do not grow with $m$. Since $\bar H$ is bijective and equals the identity on each modification ball, no point outside those balls is mapped into them. Thus $\widetilde H=\bar H$ on $\mathcal W\times\R$.

Since $\widetilde H$ is onto and preserves its first two coordinates, $\widetilde T$ is onto. For $q=(q_a,q_b)\in\R^4$,
\[
 D\widetilde T(y)^\trans q
 =D\widetilde H(y)^\trans(\sqrt2q_a,q_b,0),
 \qquad
 \norm{D\widetilde T(y)^\trans q}\ge c\norm q,
\]
where the last inequality uses the uniform inverse-derivative bound for $\widetilde H$. \Cref{lem:completion-square} now gives $\max_y\tilde{g}_{\mathrm{cycle}}(u,y)=\norm u^2/2$ and global inner P{\L} with a universal constant.

To verify joint smoothness, write $K:=\pi_b\widetilde H$ for the last two residual coordinates. The derivative bounds above control $DK$ and $D^2K$ uniformly. Wherever $D^2K$ is nonzero, $K$ is bounded as well: outside the modification balls this follows from the fixed compact support in $(a,b)$ of the original residual's nonlinear part, and inside a ball it follows because the ball maps onto itself. In particular, its $b$-coordinates stay near $-Be_1$, regardless of the scalar center $r_t$. Thus
\[
 \left\|\nabla^2\!\left(\frac12\norm K^2\right)\right\|
 \le \norm{DK}^2+\norm K\,\norm{D^2K}\le C.
\]
Expanding \eqref{eq:low-interpolated-game-app} gives $\tilde{g}_{\mathrm{cycle}}(u,a,b,r)=\sqrt2u^\trans a-\norm a^2-\norm{K(a,b,r)}^2/2$. Since $a$ is unchanged, the outer-dependent term remains bilinear, so the full Hessian is uniformly bounded. On $\mathcal W\times\R$, we have $K(a,b,r)=R(a,b)$; hence $\tilde g_{\mathrm{cycle}}$ agrees with $g_{\mathrm{cycle}}$ for every $u$ and is independent of $r$ there.

Finally, \eqref{eq:low-derivative-gradient-app} gives the target gradients in \eqref{eq:low-target-gradients-app}. At each initial query, the zero primal gradient leaves $u=e_1$ unchanged, so the simultaneous and alternating methods use the same inner-gradient query and both move from $y_t$ to $y_{t+1}$. The initial inner gap is $(1+B^2)/2$ and the initial primal gradient is zero.

At time $m$, both methods reach $(e_1,y_\star(0),0)$. Agreement with the base game on $\mathcal W\times\R$ covers both update orders, including the alternating inner-gradient query after the primal update. Independence of $r$ keeps $r_t=0$, while \cref{lem:low-periodic-solution} keeps $(u_t,a_t,b_t)\in\mathcal U$ under every small-step continuation. For either $m=0$ or $m\ge1$, the construction gives $\norm{\nabla_u\tilde g_{\mathrm{cycle}}(e_1,y_0)}\le1$. All constants depend only on the fixed small-step construction, not on the number or values of the prescribed steps.
\end{proof}

%% file: appendices/smoothed_comparison.tex
\Needspace{23\baselineskip}
\section{Comparison with Smoothed-AGDA}
\label{app:smoothing-alignment}

We compare the expected squared outer gradient at a uniformly selected SGDA iterate with that at a point computed from a Smoothed-AGDA iterate pair. Algorithm~\ref{alg:smoothed-agda} states the updates of \citet[Algorithm~2]{yang2022faster}, with an auxiliary center $z_t$ initialized at $x_0$ and regularization coefficient $2\ell$.

\begin{algorithm}[H]
\caption{Stochastic Smoothed-AGDA}
\label{alg:smoothed-agda}
\begin{algorithmic}[1]
\Require Initial pair $(x_0,y_0)$, integer $N\ge1$, and step $0<\alpha\le1/(3\ell)$.
\State Set $z_0=x_0$ and $\beta=\alpha\ell/(76800\kappa)$.
\For{$t=0,\ldots,N-1$}
 \State Draw fresh independent samples $\xi_t^x,\xi_t^y$.
 \State $x_{t+1}=x_t-\alpha\bigl[G_x(x_t,y_t;\xi_t^x)+2\ell(x_t-z_t)\bigr]$.
 \State $y_{t+1}=y_t+(\alpha/48)G_y(x_{t+1},y_t;\xi_t^y)$.
 \State $z_{t+1}=z_t+\beta(x_{t+1}-z_t)$.
\EndFor
\State Draw $I$ uniformly from $\{0,\ldots,N-1\}$, independently of the oracle samples.
\State \Return $(\widetilde x,\widetilde y)=(x_I,y_I)$.
\end{algorithmic}
\end{algorithm}

Smoothed-AGDA controls the squared game gradients, but directly bounding the squared outer gradient introduces an extra factor of $\kappa$ through the inner-gradient error. The exact minimizer of the regularized outer objective avoids this loss, and stochastic updates approximate it. We apply the dual-first updates \eqref{eq:smoothed-conversion-updates} to $f(x,y)+\ell\norm{x-\widetilde x}^2$, starting from $(\widetilde x,\widetilde y)$. The guarantee below combines the mean-square estimates of \citet[Theorem~4.1, Proposition~2.1(b), and Corollary~4.1]{yang2022faster} with strong convexity in $x$.

\begin{proposition}[Outer stationarity for Smoothed-AGDA]
\label{prop:smoothed-comparison}
Suppose Assumptions~\ref{ass:smoothness}--\ref{ass:bounded-variance} hold with $\kappa\ge2$, and assume in addition that
\begin{equation}
 \Phi(x_0)-\inf_x\{f(x,y_0)+\ell\norm{x-x_0}^2\}\le c_{\rm init}\Delta.
 \label{eq:smoothed-extra-initialization}
\end{equation}
For every $0<\eps^2\le\ell\Delta$, Smoothed-AGDA (Algorithm~\ref{alg:smoothed-agda}) followed by the dual-first conversion in \eqref{eq:smoothed-conversion-updates} returns a point $\widehat x$ with $\E\norm{\nabla\Phi(\widehat x)}^2\le\eps^2$, using at most
\begin{equation}
 C\left[
 \frac{\kappa\ell\Delta}{\eps^2}
 +\frac{\kappa^2\ell\Delta\sigma^2}{\eps^4}
 +\left(\kappa+\frac{\kappa^4\sigma^2}{\eps^2}\right)\log(2\kappa)
 \right]
 \label{eq:smoothed-total-complexity}
\end{equation}
stochastic partial-gradient calls, where $C$ depends at most on $c_{\rm init}$. In particular, if $\eps^2\le c\ell\Delta/(\kappa^2\log(2\kappa))$ for a sufficiently small universal $c>0$, this cost is
\begin{equation}
 O\!\left(\frac{\kappa\ell\Delta}{\eps^2}
 +\frac{\kappa^2\ell\Delta\sigma^2}{\eps^4}\right).
 \label{eq:smoothed-small-error-complexity}
\end{equation}
\end{proposition}

\begin{proof}
All step sizes and iteration counts below depend only on the accuracy and problem parameters, not on the realized pair.

\paragraph{The iterate pair from Smoothed-AGDA.}
Initialize the auxiliary center of \citet[Algorithm~2]{yang2022faster} at $z_0=x_0$. Their initial quantity $b_0$ is twice the left side of \eqref{eq:smoothed-extra-initialization}, so $b_0\le2c_{\rm init}\Delta$. Before tuning the primal step $\alpha$, the proof of their Theorem~4.1, Eqs.~(40)--(41), bounds the average squared game gradients by
\begin{equation}
 \frac1N\sum_{t=0}^{N-1}\E\!\left[
 \norm{\nabla_xf(x_t,y_t)}^2+\kappa\norm{\nabla_yf(x_t,y_t)}^2\right]
 \le C\frac{\kappa\Delta}{N\alpha}+C\kappa\ell\alpha\sigma^2,
 \qquad 0<\alpha\le\frac1{3\ell}.
 \label{eq:smoothed-pair-rate}
\end{equation}
Here the dual step is $\alpha/48$, the regularization coefficient is $2\ell$, and the center relaxation is $\alpha\ell/(76800\kappa)$, as in their proof. Their descent estimates use only the conditional means and variances of the two gradient calls; no smoothness of individual sampled gradients is required. They therefore apply to the oracle formulation in Assumptions~\ref{ass:unbiasedness} and~\ref{ass:bounded-variance}.

Choose $\alpha=\min\{1/(3\ell),\sqrt{\Delta/(N\ell\sigma^2)}/2\}$, omitting the second term when $\sigma=0$. With $N\ge C(\kappa\ell\Delta/\eps^2+\kappa^2\ell\Delta\sigma^2/\eps^4)$, a uniformly selected pair $(\widetilde x,\widetilde y)$ satisfies $\E e\le\eps^2/16$, where $e:=\norm{\nabla_xf(\widetilde x,\widetilde y)}^2+\kappa\norm{\nabla_yf(\widetilde x,\widetilde y)}^2$. Using the upper budget $\Delta$ in this choice, rather than the actual initial outer gap, also covers initializations with zero outer gap.

\paragraph{The exact proximal point.}
The gradient-bias bound in Lemma~\ref{lem:inner-pl-consequences} would only give $\norm{\nabla\Phi(\widetilde x)}^2\le2\kappa e$. Instead, condition on the selected pair and define the auxiliary game $\widehat f(x,y):=f(x,y)+\ell\norm{x-\widetilde x}^2$. It is $\ell$-strongly convex in $x$, retains inner P{\L} with parameter $\ell/\kappa$, and is $3\ell$-smooth in the partial-gradient convention of Assumption~\ref{ass:smoothness}. Since $f(\cdot,y)+\ell\norm{\cdot}^2/2$ is convex for every $y$, its supremum has the same property. Thus $\Phi(x)+\ell\norm{x-\widetilde x}^2$ is $\ell$-strongly convex and has a unique minimizer $p:=\operatorname*{arg\,min}_x\{\Phi(x)+\ell\norm{x-\widetilde x}^2\}$.
Inner P{\L} and strong convexity in $x$ give
\begin{align}
 \Phi(\widetilde x)-\min_x\widehat f(x,\widetilde y)
 &=\Phi(\widetilde x)-f(\widetilde x,\widetilde y)
   +\widehat f(\widetilde x,\widetilde y)-\min_x\widehat f(x,\widetilde y)
 \le\frac{e}{2\ell},\notag\\
 \frac\ell2\norm{p-\widetilde x}^2
 &\le\Phi(\widetilde x)-\Phi(p)-\ell\norm{p-\widetilde x}^2
 \le\frac{e}{2\ell}.
 \label{eq:smoothed-proximal-gap}
\end{align}
The second line uses strong convexity of the auxiliary outer objective, then $\Phi(p)+\ell\norm{p-\widetilde x}^2\ge\min_x\widehat f(x,\widetilde y)$. Optimality of $p$ now yields the pointwise bound
\begin{equation}
 \nabla\Phi(p)=2\ell(\widetilde x-p),\qquad
 \norm{\nabla\Phi(p)}^2\le4e.
 \label{eq:smoothed-exact-proximal-gradient}
\end{equation}
This is the squared-gradient estimate underlying \citet[Proposition~2.1(b), Eq.~(7)]{yang2022faster}; no conversion from a first-moment norm bound is used.

\paragraph{A finite stochastic conversion.}
Starting at $(x^0,y^0)=(\widetilde x,\widetilde y)$, apply the dual-first alternating method used in that proposition to $\widehat f$, with dual step $\tau$ and primal step $162\tau$. Explicitly, use fresh independent samples to update
\begin{align}
 y^{k+1}&=y^k+\tau G_y(x^k,y^k;\xi_k^y),\notag\\
 x^{k+1}&=x^k-162\tau\bigl[G_x(x^k,y^{k+1};\xi_k^x)+2\ell(x^k-\widetilde x)\bigr],
 \qquad 0<\tau\le\frac1{486\ell}.
 \label{eq:smoothed-conversion-updates}
\end{align}
These are an additional postprocessing procedure, not the primal-first recursion defining \SGDAAlt{} in Section~\ref{sec:setup}.

We use the potential contraction of \citet[Theorem~A.1 and the proof of Proposition~2.1(b)]{yang2022faster}, then obtain the required distance bound directly from strong convexity. Write $\Psi(y):=\min_x\widehat f(x,y)$ and $\Psi^*:=\max_y\Psi(y)$. The auxiliary game has a saddle with outer coordinate $p$: any inner maximizer at $p$ also minimizes $\widehat f(\cdot,y)$ there, by \cref{lem:inner-pl-consequences} and strong convexity. The potential combines the dual gap with the error in minimizing $\widehat f(\cdot,y^k)$:
\[
 P_k:=\E_{\rm conv}\!\left[\Psi^*-\Psi(y^k)+\frac1{10}\bigl(\widehat f(x^k,y^k)-\Psi(y^k)\bigr)\right],
\]
where $\E_{\rm conv}$ averages the fresh conversion samples conditional on $(\widetilde x,\widetilde y)$. The initialization estimate and their contraction give
\begin{equation}
 P_0\le\frac{11}{10}\left[\Phi(\widetilde x)-\min_x\widehat f(x,\widetilde y)\right]\le\frac{11e}{20\ell},
 \qquad
 P_k\le\left(1-\frac{\ell\tau}{2\kappa}\right)^kP_0+C\kappa\tau\sigma^2.
 \label{eq:smoothed-conversion-potential}
\end{equation}
To bound distance to $p$, let $x_y:=\arg\min_x\widehat f(x,y)$. Strong convexity at this minimizer and the saddle inequality $\widehat f(p,y)\le\Psi^*$ give
\begin{equation}
 \norm{x-p}^2
 \le 2\norm{x-x_y}^2+2\norm{x_y-p}^2
 \le\frac4\ell\bigl[\widehat f(x,y)-\Psi(y)+\Psi^*-\Psi(y)\bigr].
\label{eq:conversion-distance-direct}
\end{equation}
Hence $\E_{\rm conv}\norm{x^k-p}^2\le40P_k/\ell$.
The step choices satisfy their conditions with smoothness $3\ell$ and P{\L} parameters $\ell$ and $\ell/\kappa$. In particular, their variance term is $O(\kappa\tau\sigma^2)$. Substitution and expectation over the selected pair give
\begin{equation}
 \E\norm{x^k-p}^2
 \le\frac{C}{\ell^2}
       \left(1-\frac{\ell\tau}{2\kappa}\right)^k\E e
       +\frac{C\kappa\tau\sigma^2}{\ell}.
 \label{eq:smoothed-conversion-distance}
\end{equation}
By Lemma~\ref{lem:inner-pl-consequences}, the outer gradient is $2\kappa\ell$-Lipschitz. To keep the conversion's contribution to the squared outer gradient below $\eps^2/2$, it therefore suffices to make the squared distance at most $\eps^2/(16\kappa^2\ell^2)$. Combining \eqref{eq:smoothed-exact-proximal-gradient} and \eqref{eq:smoothed-conversion-distance} gives
\begin{align}
 \E\norm{\nabla\Phi(x^k)}^2
 &\le8\E e+8\kappa^2\ell^2\E\norm{x^k-p}^2\notag\\
 &\le8\E e+C\kappa^2\left(1-\frac{\ell\tau}{2\kappa}\right)^k\E e
       +C\kappa^3\ell\tau\sigma^2.
 \label{eq:smoothed-conversion-gradient}
\end{align}
The last term determines the conversion step: for $\sigma>0$, choose $\tau=\min\{1/(486\ell),c\eps^2/(\kappa^3\ell\sigma^2)\}$ with a sufficiently small universal $c>0$; use $\tau=1/(486\ell)$ when $\sigma=0$.

\Needspace{10\baselineskip}
After
\[
 k=\left\lceil\frac{C\kappa}{\ell\tau}\log(2\kappa)\right\rceil
 =O\!\left(\left[\kappa+\frac{\kappa^4\sigma^2}{\eps^2}\right]\log(2\kappa)\right)
\]
updates, \eqref{eq:smoothed-conversion-distance} meets this distance target. The factor $\kappa^2$ multiplying the contracting term in \eqref{eq:smoothed-conversion-gradient} accounts for the logarithm. This choice works in expectation over the random starting pair, without requiring its realized game gradients to satisfy a deterministic accuracy bound. With $\widehat x=x^k$, the first line of \eqref{eq:smoothed-conversion-gradient} concludes
\begin{equation}
 \E\norm{\nabla\Phi(\widehat x)}^2
 \le8\E e+\frac{\eps^2}{2}\le\eps^2.
 \label{eq:smoothed-mean-square-output}
\end{equation}
There are two partial-gradient calls per update in both stages. Adding their costs proves \eqref{eq:smoothed-total-complexity}. Under the stated small-error condition, the term $\kappa\log(2\kappa)$ is absorbed by $\kappa\ell\Delta/\eps^2$, and $\kappa^4\sigma^2\eps^{-2}\log(2\kappa)$ is absorbed by $\kappa^2\ell\Delta\sigma^2\eps^{-4}$, giving \eqref{eq:smoothed-small-error-complexity}.
\end{proof}

\subsection{Initialization budget of the hard instances}
\label{app:smoothed-initialization}

The additional budget \eqref{eq:smoothed-extra-initialization} consists of the initial inner gap and the decrease obtainable by minimizing $f(x,y_0)+\ell\norm{x-x_0}^2$ over $x$ while keeping $y_0$ fixed. We verify that their sum is at most $c_{\rm init}\Delta$ for every hard instance in Theorem~\ref{thm:common-lower}, using the same declared budget $\Delta$ and its original initialization. The ordinary outer and inner gap checks remain in the construction proofs.

\paragraph{The unstable quadratic.}
For $f_{\mathrm{unstable}}$ in \eqref{eq:quadratic-main}, Appendix~\ref{app:quadratic} chooses $x_0=4\sqrt{\Delta/(\ell+\mu)}$ and $y_0=0$. Thus $f_{\mathrm{unstable}}(x,y_0)=0$ for every $x$, and
\[
 \Phi_{\mathrm{unstable}}(x_0)-\inf_x\{f_{\mathrm{unstable}}(x,y_0)+\ell(x-x_0)^2\}
 =\Phi_{\mathrm{unstable}}(x_0)=\Delta\le c_{\rm init}\Delta.
\]
This verification includes the one-iterate case $T=1$, for which the high-ratio proof uses this quadratic.

\paragraph{The corridor and noise components.}
The game $f_{\mathrm{stable}}$ in \eqref{eq:high-separable-game-main} is initialized at $(u_0,v_0;y_0)=(0,0;0)$. Equation~\eqref{eq:rescaled-perturbation-block-app} and the zero-component convention give $g_{\sigma,j}(v[j];0)=0$ for every component, so $f_{\mathrm{stable}}(u,v;0)=g_{\mathrm{corridor}}(u)$. Also $\Phi_{\mathrm{stable}}(0,0)=g_{\mathrm{corridor}}(0)$, and \eqref{eq:one-dimensional-gap-app} gives $g_{\mathrm{corridor}}(0)-\inf_u g_{\mathrm{corridor}}(u)=\Delta$. Therefore,
\begin{align*}
 &\Phi_{\mathrm{stable}}(0,0)-\inf_{u,v}\{f_{\mathrm{stable}}(u,v;0)+\ell(u^2+\norm v^2)\}\\
 &\qquad=g_{\mathrm{corridor}}(0)-\inf_u\{g_{\mathrm{corridor}}(u)+\ell u^2\}
 \le\Delta\le c_{\rm init}\Delta.
\end{align*}
This also verifies the corridor alone, including $\sigma=0$ and schedules with no positive noise-component steps, when all noise components vanish. Each noise component alone has regularized initial gap zero.

\paragraph{The low-ratio game.}
Retain the constants $C,C_{\mathrm{low}},\zeta$ and scales $A,s_x,s_y$ from Appendix~\ref{app:low-ratio-theorem}, with $f_{\mathrm{cycle}}$ defined by \eqref{eq:low-rescaled-game-app} and initialized at $x_0=s_xe_1$, $y_0^{\mathrm{cycle}}=s_y y_0$. The construction in Appendix~\ref{app:low-ratio-interpolation} gives $\norm{\nabla_u\tilde g_{\mathrm{cycle}}(e_1,y_0)}\le1$, so $\norm{\nabla_x f_{\mathrm{cycle}}(x_0,y_0^{\mathrm{cycle}})}\le A/s_x$. Affinity in $x$, the gap bound \eqref{eq:low-gap-budgets-app}, and the scales in \eqref{eq:low-exact-scales-app} give
\[
\begin{aligned}
 &\Phi_{\mathrm{cycle}}(x_0)-\inf_x\{f_{\mathrm{cycle}}(x,y_0^{\mathrm{cycle}})+\ell\norm{x-x_0}^2\}\\
 &\qquad=\Phi_{\mathrm{cycle}}(x_0)-f_{\mathrm{cycle}}(x_0,y_0^{\mathrm{cycle}})
   +\frac{\norm{\nabla_x f_{\mathrm{cycle}}(x_0,y_0^{\mathrm{cycle}})}^2}{4\ell}\\
 &\qquad\le\left(2C+\frac{1}{2C_{\mathrm{low}}}\right)\zeta\Delta
 \le3C\zeta\Delta\le c_{\rm init}\Delta.
\end{aligned}
\]
The last inequalities use $C\ge1$, $C_{\mathrm{low}}\ge2C$, and the choice $\zeta\le c_{\rm init}/(3C)$ already made in Appendix~\ref{app:low-ratio-theorem}. The check is uniform over the horizon and schedule, including $T=1$. Both this game and the unstable quadratic use exact gradients, so their checks apply for every $\sigma\ge0$.

\paragraph{Scope of the separation.}
All hard instances satisfy \eqref{eq:smoothed-extra-initialization}, so the SGDA lower bound persists on the subclass for which \cref{prop:smoothed-comparison} applies. Under its small-error condition, the complexity comparison uses the same declared gap budget $\Delta$. SGDA returns a uniformly selected iterate, whereas $\widehat x$ is computed by additional iterations.

%% file: appendices/table_comparison.tex
\section{Complexity calculations for Table~\ref{tab:ncpl-landscape}}
\label{app:table-comparison}

We derive the entries of Table~\ref{tab:ncpl-landscape}, retaining initialization and additional source parameters until the final specialization. For the upper bounds, assume that the initial outer and inner gaps are $O(\Delta)$; the table fixes $\Delta$. MSGDA and AdaMSGDA require every sampled gradient to be $\ell$-Lipschitz, while SPIDER-GDA and SGDA-RR require $\ell$-smooth finite-sum components. SGDA and Smoothed-AGDA require only population smoothness and bounded oracle variance. Expected-norm criteria use tolerance $\eps$, and expected-squared-norm criteria use $\eps^2$. The complete complexities below specify the regimes for the MSGDA and AdaMSGDA specializations and distinguish HCMM-2's noise contribution from its full cost.

\subsection{SGDA and Smoothed-AGDA}

\paragraph{Simultaneous and alternating SGDA.}
Under the setup of this paper, \citet[Theorem~3.1]{yang2022faster} for \SGDAAlt{} and \cref{thm:sim-upper} for \SGDASim{} give
\begin{align}
 \frac1T\sum_{t<T}\E\norm{\nabla\Phi(x_t)}^2
 &\le C\left[\frac{\kappa^2\ell\Delta}{T}
       +\kappa^2\sigma\sqrt{\frac{\ell\Delta}{T}}\right],\notag\\
 T&=O\left(1+\frac{\kappa^2\ell\Delta}{\eps^2}
       +\frac{\kappa^4\ell\Delta\sigma^2}{\eps^4}\right).
 \label{eq:table-sgda-complexity}
\end{align}
The second line makes each contribution in the first at most $\eps^2/2$. Each iteration uses two fresh partial-gradient calls, so iteration and oracle complexities differ by a constant factor. Fixing $\Delta$ and taking the usual small-error regime $\eps^2\le c\ell\Delta$ gives the two upper-bound entries. The matching lower bound in \cref{thm:common-lower} uses the same squared outer-gradient criterion and optimizes over fixed timescale ratios and predetermined nonincreasing step sizes.

\paragraph{Smoothed-AGDA.}
Proposition~\ref{prop:smoothed-comparison}, derived from \citet[Theorem~4.1, Proposition~2.1(b), and Corollary~4.1]{yang2022faster}, includes the conversion to a point $\widehat x$ satisfying $\E\norm{\nabla\Phi(\widehat x)}^2\le\eps^2$. Under the additional regularized initialization condition \eqref{eq:smoothed-extra-initialization}, its total oracle cost is
\begin{equation}
 O\left(\frac{\kappa\ell\Delta}{\eps^2}
       +\frac{\kappa^2\ell\Delta\sigma^2}{\eps^4}
       +\left(\kappa+\frac{\kappa^4\sigma^2}{\eps^2}\right)\log(2\kappa)\right).
 \label{eq:table-smoothed-complexity}
\end{equation}
If $\eps^2\le c\ell\Delta/(\kappa^2\log(2\kappa))$, then
$\kappa^4\sigma^2\eps^{-2}\log(2\kappa)\lesssim
\kappa^2\ell\Delta\sigma^2\eps^{-4}$ and
$\kappa\log(2\kappa)\lesssim\kappa\ell\Delta\eps^{-2}$.
These absorptions give the table's $\widetilde O(\kappa\ell\eps^{-2}+\kappa^2\ell\sigma^2\eps^{-4})$ entry with $\Delta$ fixed. The returned point is computed by the conversion in Appendix~\ref{app:smoothing-alignment}, rather than selected directly from the original trajectory.

\subsection{Finite-sum upper bounds}

\paragraph{SPIDER-GDA.}
For $f=n^{-1}\sum_{i=1}^n f_i$ with $\ell$-smooth components and attained inner $\mu$-P{\L} maximization, \citet[Theorem~6.1 and Lemma~D.2]{chen2022faster} use $\kappa=\ell/\mu$ and initial potential
$V_0=\Phi(x_0)-\inf\Phi+(\Phi(x_0)-f(x_0,y_0))/24$.
With primal step $\tau_x=\Theta((\kappa^2\ell)^{-1})$, their proof gives a uniformly selected iterate $\widehat x$ with
\begin{align}
 \E\norm{\nabla\Phi(\widehat x)}^2
 &\le \frac{2V_0}{K\tau_x}
 \le \frac{C\kappa^2\ell V_0}{K},\notag\\
 \text{component-gradient calls}
 &=O(n+K\sqrt n)
 =O\left(n+\frac{\sqrt n\,\kappa^2\ell V_0}{\eps^2}\right).
 \label{eq:table-spider-complexity}
\end{align}
Here we choose $K=\max\{1,\lceil C\kappa^2\ell V_0/\eps^2\rceil\}$ updates, and the source chooses both the refresh period and recursive-estimator batch size of order $\sqrt n$. A full gradient costs $n$ component calls, which also absorbs the rounding cost. The squared bound gives the source's stated expected-norm criterion by Jensen's inequality; $V_0=O(\Delta)$ gives the table entry. No uniform component-variance bound is needed.

\paragraph{SGDA-RR.}
\citet[Theorem~1]{cho2023sgda} allow a partial-gradient component variance bound of the form
$n^{-1}\sum_i\norm{\nabla_j f_i(z)-\nabla_j f(z)}^2
\le A\norm{\nabla_j f(z)}^2+B$, for $j\in\{x,y\}$.
The table specializes to $A=0$ and $B=\sigma^2$. With ratio $\rho\ge14\kappa^2$, $K$ epochs, and
$V_0=4(\Phi(x_0)-\inf\Phi)+\Phi(x_0)-f(x_0,y_0)$, their bound gives
\begin{align}
 \frac1K\sum_{k=1}^K\E\norm{\nabla\Phi(x_0^k)}^2
 &\le C\left[\frac{\rho\ell V_0}{K}
       +\rho\left(\frac{\ell^2\sigma^2V_0^2}{nK^2}\right)^{1/3}\right],\notag\\
 nK&=O\left(n+\frac{n\rho\ell V_0}{\eps^2}
       +\frac{\sqrt n\,\rho^{3/2}\ell V_0\sigma}{\eps^3}\right).
 \label{eq:table-rr-complexity}
\end{align}
The second line makes both error terms $O(\eps^2)$ and counts $n$ component calls per epoch, up to the constant factor for the two partial gradients; its first term covers rounding to a positive integer number of epochs. Substituting $\rho=\Theta(\kappa^2)$ and $V_0=O(\Delta)$, then using $\eps^2\le c\ell\Delta$ to absorb this rounding term, gives the table's two terms. The output is sampled from epoch starts. Component reuse and sampling without replacement differ from fresh-sample SGDA.

\subsection{MSGDA and AdaMSGDA}

\paragraph{MSGDA.}
We specialize \citet[Theorem~1]{huang2025enhanced} to zero nonsmooth regularizer and unconstrained variables. Assume that each sampled partial gradient is unbiased with variance at most $\sigma^2$, and that every sampled joint gradient is $\ell$-Lipschitz. The source's full-gradient variance bound is then $2\sigma^2$, while its $L_f=\ell$ and $\kappa=\ell/\mu$. We use the valid outer-smoothness bound $L=\ell(1+\kappa)$ from \cref{lem:inner-pl-consequences} in place of the smaller constant printed in the source. Each update draws one fresh sample and evaluates that same sampled function at two iterates, using a constant number of gradient calls.

The source's proof controls the primal estimator error, the inner gap, and the primal update direction. With $w_t$ denoting its primal gradient estimator, the inner distance estimate in the proof of \cref{lem:inner-pl-consequences} gives
\begin{equation}
 \norm{\nabla\Phi(x_t)}^2
 \le 3\norm{w_t-\nabla_x f(x_t,y_t)}^2
     +6\kappa\ell[\Phi(x_t)-f(x_t,y_t)]
     +3\norm{w_t}^2.
 \label{eq:table-msgda-gradient-transfer}
\end{equation}
Indeed, smoothness bounds the squared difference between the outer gradient and the game gradient by $2\kappa\ell$ times the inner gap. With no regularizer, the source's primal gradient mapping equals $w_t$. Thus its averaged squared estimates transfer directly to the squared outer-gradient criterion.

Write $D_x=\Phi(x_1)-\inf\Phi$, $D_y=\Phi(x_1)-f(x_1,y_1)$, and $D=D_x+D_y$. The accumulated estimate below bounds the initial gaps and initial estimator variance by a term proportional to $(\kappa\ell D+\kappa^2\sigma^2)/k$, and the accumulated estimator variance by a term proportional to $\kappa^4k^2\sigma^2$ up to logarithms. Balancing them for $\sigma>0$ gives the following choice of the source's parameters:
\begin{align}
 k&=\left(\frac{\kappa\ell D+\kappa^2\sigma^2}{\kappa^4\sigma^2}\right)^{1/3},
 &c_1&=\frac{2}{3k^3}+9\kappa^2,
 &c_2&=\frac{2}{3k^3}+\frac94,\notag\\
 m&=(c_1k)^3,
 &\lambda&=\frac{\kappa}{8\ell},
 &\gamma&=\frac1{256\kappa\ell}.
 \label{eq:table-msgda-tuning}
\end{align}
These choices satisfy the source proof's restrictions for every $\ell>0$ and $\kappa\ge1$: $k\ge\kappa^{-2/3}$, $9\kappa^2\le c_1\le(29/3)\kappa^2$, and $c_2\le c_1$ give all lower bounds on $m$. In particular, the restriction involving the corrected outer-smoothness constant holds because
\begin{equation*}
 \gamma=\frac1{256\kappa\ell}
 \le\frac1{128\ell(1+\kappa)}
 =\frac{\lambda\mu}{16L}.
\end{equation*}
The remaining primal and dual step restrictions are unchanged and are met by the displayed choices. The schedule is $\eta_t=k/(m+t)^{1/3}$, with momentum weights $c_1\eta_t^2$ and $c_2\eta_t^2$.

The source's coefficient $H$ and the squared-gradient estimate obtained from its proof equations~(59) and~(61) satisfy
\begin{align}
 H&=\frac{D_x}{\gamma k}+\frac{9\kappa^2D_y}{k\lambda}
     +\frac{4\sigma^2m^{1/3}}{k^2}
     +4k^2(c_1^2+c_2^2)\sigma^2\log(m+T)\notag\\
 &\le C\left[\frac{\kappa\ell D+\kappa^2\sigma^2}{k}
       +\kappa^4k^2\sigma^2\log(m+T)\right],\notag\\
 \frac1T\sum_{t=1}^T\E\norm{\nabla\Phi(x_t)}^2
 &\le \frac{12H(m+T)^{1/3}}{T}
 \le12H\left(\frac{m^{1/3}}T+\frac1{T^{2/3}}\right).
 \label{eq:table-msgda-error}
\end{align}
The last line combines the source's averaged squared estimates with \eqref{eq:table-msgda-gradient-transfer}. Making its two contributions at most $\eps^2/2$ requires $T=O(1+Hm^{1/3}\eps^{-2}+H^{3/2}\eps^{-3})$. Substituting the balanced choice of $k$ gives
\begin{align}
 Hm^{1/3}&=\widetilde O(\kappa^3\ell D+\kappa^4\sigma^2),
 &H^{3/2}&=\widetilde O(\kappa^3\ell D\sigma+\kappa^4\sigma^3),\notag\\
 T&=\widetilde O\left(
 1+\frac{\kappa^3\ell D+\kappa^4\sigma^2}{\eps^2}
 +\frac{\kappa^3\ell D\sigma+\kappa^4\sigma^3}{\eps^3}\right).
 \label{eq:table-msgda-complexity}
\end{align}
This is a feasible specialization derived from the cited theorem, rather than the source's accuracy-only statement. For $0<\eps\le\sigma$ and fixed $D$, the $\eps^{-3}$ terms dominate the $\eps^{-2}$ terms, yielding the table's $\widetilde O((\kappa^3\ell\sigma+\kappa^4\sigma^3)\eps^{-3})$ entry. At $\sigma=0$, instead take $k=\max\{\kappa^{-2/3},(\ell D/(\kappa^3\eps^2))^{1/3}\}$ with the remaining choices in \eqref{eq:table-msgda-tuning}; the same estimate gives $O(1+\kappa^3\ell D\eps^{-2})$.

\paragraph{AdaMSGDA.}
Theorem~2 of \citet{huang2025enhanced} additionally assumes
$A_t\succeq aI$ and $b_-I\preceq B_t\preceq b_+I$, where $a>0$ and $0<b_-\le b_+$ are fixed spectral constants. To apply the source's dual-update estimate, enlarge the upper bound to $\bar b_+=\max\{b_+,b_-^{-1}\}$. Then $B_t\preceq\bar b_+I$ and $b_-\bar b_+\ge1$. For its dual gradient estimator $v_t$, the update $\widetilde y_{t+1}=y_t+\lambda B_t^{-1}v_t$ satisfies
\begin{equation*}
 v_t^\top B_t^{-1}v_t
 \ge b_-\norm{B_t^{-1}v_t}^2
 \ge\frac{\norm{\widetilde y_{t+1}-y_t}^2}{\lambda^2\bar b_+}.
\end{equation*}
This supplies the coefficient used in equation~(66) of the source proof with $\bar b_+$ as the spectral upper bound.

Use the same $k,c_2$ as above, but take
$c_1=2/(3k^3)+9(\bar b_+/b_-)\kappa^2$ and $m=(c_1k)^3$.
For explicit admissible steps, set $d=1/(8\bar b_+)$ and
$g=\min\{d/(32\bar b_+),b_-/(16\bar b_+),a/8,a^2d/(9\bar b_+)\}$, and choose
$\lambda=d\kappa/\ell$ and $\gamma=g/(\kappa\ell)$.
The outer-smoothness restriction again holds with $L=\ell(1+\kappa)$, since $g\le d/(32\bar b_+)$ gives
\begin{equation*}
 \gamma\le\frac{d}{32\bar b_+\kappa\ell}
 \le\frac{d}{16\bar b_+\ell(1+\kappa)}
 =\frac{\lambda\mu}{16\bar b_+L}.
\end{equation*}
The other terms defining $g$ satisfy the remaining primal step restrictions, while $m^{1/3}/k=c_1$ satisfies the schedule-dependent primal and dual restrictions. Its coefficient $G$ obeys
\begin{align}
 G&=\frac{D_x}{\gamma ka}
     +\frac{9\bar b_+\kappa^2D_y}{k\lambda a^2}
     +\frac{4\sigma^2m^{1/3}}{k^2a^2}
     +\frac{4k^2(c_1^2+c_2^2)\sigma^2}{a^2}\log(m+T)\notag\\
 &\le C_{a,b_-,b_+}\left[
       \frac{\kappa\ell D+\kappa^2\sigma^2}{k}
       +\kappa^4k^2\sigma^2\log(m+T)\right].
 \label{eq:table-adamsgda-coefficient}
\end{align}
Equations~(96)--(97) of the source proof give
$T^{-1}\sum_{t=1}^T\E\norm{A_t^{-1}\nabla\Phi(x_t)}^2
\le12G(m+T)^{1/3}/T$.
Thus \eqref{eq:table-msgda-complexity} and its table specialization hold for this adaptive gradient mapping, with constants depending on $a,b_-,b_+$. A Euclidean outer-gradient guarantee additionally requires $A_t\preceq a_+I$: then
$\norm{\nabla\Phi(x_t)}\le a_+\norm{A_t^{-1}\nabla\Phi(x_t)}$, so replace the mapping tolerance by $\eps/a_+$. The source's lower bound on $A_t$ alone does not supply this conversion. Both sampled-gradient methods therefore require stronger oracle information than the bounded-variance fresh-sample model of this paper.

\subsection{HCMM-2}

\citet[Theorem~2]{cai2026bias} control the average expected outer-gradient norm. Their assumptions require attained inner P{\L} maximization, population-gradient smoothness $\ell$, population-Hessian Lipschitzness $M$, and stochastic gradient and Hessian oracles based on fresh independent samples across iterations, with conditional unbiasedness and variances at most $\sigma^2$ and $\sigma_h^2$, respectively. The parameters $M$ and $\sigma_h$ are not determined by $\ell,\kappa,\sigma$. Initialization uses one stochastic gradient; each subsequent iteration uses a constant number of gradient and Hessian-vector operations.

Write $D_x=\Phi(x_0)-\inf\Phi$, $D_y=\Phi(x_0)-f(x_0,y_0)$, and $D=D_x+D_y$. For momentum weight $b\in(0,1]$ and normalized step lengths $a_y=a>0$ and $a_x=a/(6\kappa)$, the source's conditions~(27) hold. Its Appendix~E, equation~(134), gives
\begin{align}
 \frac1T\sum_{t<T}\E\norm{\nabla\Phi(x_t)}
 \le C\Bigg[&
 \frac{\kappa D}{aT}+\frac{\kappa\sigma}{bT}
 +\frac{\kappa Ma^2}{b}
 +\frac{\kappa\sigma_h a}{\sqrt b}\notag\\
 &+\kappa\sigma\sqrt b+\kappa^2\ell a
 +\frac{\sqrt{\kappa\ell D_y}}T\Bigg].
 \label{eq:table-hcmm-error}
\end{align}
The last term uses the P{\L} quadratic-growth bound
$\operatorname{dist}(y_0,\argmax_y f(x_0,y))\le\sqrt{2\kappa D_y/\ell}$.
We use the larger outer-gap coefficient obtained directly from the source's descent inequality~(46); the coefficient printed in~(134) differs by a constant factor, which does not change these orders.

Choose $a=b=T^{-2/3}$. Substitution into \eqref{eq:table-hcmm-error}, followed by making each contribution at most $\eps/3$, yields
\begin{align}
 \frac1T\sum_{t<T}\E\norm{\nabla\Phi(x_t)}
 &\le C\left[
 \frac{\kappa(D+\sigma+\sigma_h)}{T^{1/3}}
 +\frac{\kappa M+\kappa^2\ell}{T^{2/3}}
 +\frac{\sqrt{\kappa\ell D_y}}T\right],\notag\\
 T&=O\left(
 1+\frac{\kappa^3(D^3+\sigma^3+\sigma_h^3)}{\eps^3}
 +\frac{(\kappa M+\kappa^2\ell)^{3/2}}{\eps^{3/2}}
 +\frac{\sqrt{\kappa\ell D_y}}{\eps}\right).
 \label{eq:table-hcmm-complexity}
\end{align}
For a uniformly selected iterate, this is a sufficient complexity for expected norm at most $\eps$, not expected squared norm at most $\eps^2$. Under $\sigma_h\le\sigma$, its noise contribution at order $\eps^{-3}$ is $O(\kappa^3\sigma^3\eps^{-3})$, giving the table's $O_\ell$ entry.

An additional sufficient regime in which this entry bounds the complete cost is $\sigma>0$, $D=O(\sigma)$, $\sigma_h\le\sigma$, and
$0<\eps\le c\min\{\sigma,\sigma^2/(\ell+M/\kappa)\}$.
The second accuracy restriction absorbs the Hessian-regularity term, and the two restrictions together with $D_y\le D$ absorb the terminal term. Hence \eqref{eq:table-hcmm-complexity} becomes $O(\kappa^3\sigma^3\eps^{-3})$ throughout this regime. For general initialization, the full bound also contains $\kappa^3D^3\eps^{-3}$; fixing the initial gaps does not make this term smaller in $\eps$. The table then displays a noise contribution, while \eqref{eq:table-hcmm-complexity} retains the initialization and Hessian-regularity terms.

\subsection{Lower bounds}

\paragraph{NC-SC oracle lower bounds.}
For deterministic linear-span methods, \citet[Theorem~3.1]{zhang2021complexity} prove $\Omega(\sqrt\kappa\,\ell\Delta\eps^{-2})$ full-gradient calls to reach outer-gradient norm at most $\eps$. Their finite-sum Theorem~3.2 gives $\Omega(n+\sqrt{n\kappa}\,\ell\Delta\eps^{-2})$ expected component calls to reach an $\eps$-stationary query for randomized linear-span methods. Its criterion is expected hitting time, as defined in their equations~(4)--(6), rather than the expected gradient norm of a specified output. This result requires $\kappa\ge2n$ and averaged component smoothness
$n^{-1}\sum_i\norm{\nabla f_i(z)-\nabla f_i(z')}^2
\le\ell^2\norm{z-z'}^2$, which is weaker than smoothness of every component. Both rates apply in the stated small-error regime $\eps^2\le c\ell\Delta$; fixing $\Delta$ gives the two table entries.

The first-order lower bound of \citet[Theorem~2]{li2021complexity} is proved for zero-respecting algorithms on nonempty closed convex domains $\mathcal X,\mathcal Y$, with $0\in\mathcal X$ and $\Phi(0)-\min_{x\in\mathcal X}\Phi(x)\le\Delta$. Its projected mapping is
$G_\Phi(x):=L_\Phi[x-P_{\mathcal X}(x-\nabla\Phi(x)/L_\Phi)]$, where $P_{\mathcal X}$ is Euclidean projection and $L_\Phi\le(\kappa+1)\ell$ is the source's value-function smoothness constant. The target is $\E\norm{G_\Phi(\widehat x)}\le\eps$, and the lower bound is
$\Omega(\sqrt\kappa\,\ell\Delta\eps^{-2}
+\kappa^{1/3}\ell\Delta\sigma^2\eps^{-4})$ calls to its bounded-variance stochastic first-order oracle. Fixing $\Delta$ gives the table entry; its criterion and algorithm class differ from the unconstrained SGDA recursion studied here.

\paragraph{NC-P{\L} oracle lower bound.}
\citet[Theorem~3.2]{pan2026lower} prove
$\Omega(\kappa\ell\Delta\eps^{-2})$ exact first-order calls for arbitrary deterministic first-order methods, with only an initial outer-gap budget. The target is ordinary outer-gradient stationarity, with $\kappa$ above a universal constant and $\eps^2\le c\ell\Delta$. Fixing $\Delta$ gives the NC-P{\L} first-order row.

\paragraph{Ratio obstructions and our matching complexity.}
\citet[Theorem~4.1]{li2022local} construct an NC-SC quadratic on which the joint iterates of simultaneous GDA fail to converge locally for every positive constant step when $\rho\le\kappa$. This implies the displayed local obstruction for $\rho=o(\kappa)$; it is not an averaged outer-gradient complexity bound.

For our update family, \cref{thm:common-lower,thm:high-ratio,thm:low-ratio} and the matching upper bounds give, respectively,
\begin{align}
 N_{\eps,\mathrm{Sim}}^\star,\ N_{\eps,\mathrm{Alt}}^\star
 &=\Theta\left(\frac{\kappa^2\ell\Delta}{\eps^2}
       +\frac{\kappa^4\ell\Delta\sigma^2}{\eps^4}\right),\notag\\
 N_{\eps,\mathrm{Sim}}(\rho),\ N_{\eps,\mathrm{Alt}}(\rho)
 &\ge c\left(\frac{\rho\ell\Delta}{\eps^2}
       +\frac{\kappa^4\ell\Delta\sigma^2}{\eps^4}\right)
       \quad\text{if }\rho\ge c_{\rm high}\kappa^2.
 \label{eq:table-own-complexities}
\end{align}
These complexity statements use $\eps^2\le c\ell\Delta$, with $\kappa$ above the respective theorem thresholds. Fixing $\Delta$ gives the optimized and high-ratio entries. For $1\le\rho\le c_{\rm low}\kappa^2$, \cref{thm:low-ratio} instead gives $\min_{t<T}\norm{\nabla\Phi(x_t)}^2\ge c\ell\Delta$ on a hard deterministic instance for every prescribed horizon and nonincreasing schedule. Thus the nonconvergent entry $\infty$ records a nonvanishing stationarity error, rather than an iteration complexity; the hard instance may depend on that horizon and schedule.